\documentclass{article}

 \usepackage[final]{nips_2018}

\usepackage[utf8]{inputenc} 
\usepackage[T1]{fontenc}    
\usepackage{hyperref}       
\usepackage{url}            
\usepackage{booktabs}       
\usepackage{amsfonts}       
\usepackage{amsmath}       
\usepackage{nicefrac}       
\usepackage{microtype}      
\usepackage{xspace}

\usepackage{mathtools}
\usepackage{mathrsfs}
\usepackage{subcaption}
\usepackage{dsfont}
\usepackage{amssymb}
\usepackage{amsthm}

\newcommand{\iid}{\textit{i.i.d.}\xspace} 
\newcommand{\EE}{\mathds{E}}
\DeclarePairedDelimiterX{\infdivx}[2]{(}{)}{
  #1\;\delimsize\|\;#2}

\def \Paren#1{{\left({#1}\right)}}
\DeclareMathOperator*{\Exp}{\EE}

\def \Ceil#1{{\left\lceil{#1}\right\rceil}}

\newcommand{\indic}{\mathds{1}}
\newtheorem{theorem}{Theorem}

\newtheorem{Theorem}{Theorem}
\newtheorem{Theorem*}{Theorem}
\newtheorem{Lemma}[Theorem]{Lemma}
\newtheorem{Lemma*}[Theorem]{Lemma}
\title{On Learning Markov Chains}

\author{
  Yi~HAO\\
  Dept. of Electrical and Computer Engineering \\
 University of California, San Diego\\
 La Jolla, CA 92093\\
  \texttt{yih179@ucsd.edu} \\
 \And
Alon~Orlitsky \\
 Dept. of  Electrical and Computer Engineering\\
 University of California, San Diego\\
 La Jolla, CA 92093\\
  \texttt{alon@ucsd.edu} \\
 \AND
Venkatadheeraj Pichapati\\
 Dept. of  Electrical and Computer Engineering\\
 University of California, San Diego\\
 La Jolla, CA 92093\\
  \texttt{dheerajpv7@ucsd.edu} \\
}

\begin{document}

\maketitle
\begin{abstract}
The problem of estimating an unknown discrete distribution from its samples is a fundamental tenet of statistical learning. Over the past decade, it attracted significant research effort and has been solved for a variety of divergence measures.  Surprisingly, an equally important problem, estimating an unknown Markov chain from its samples, is still far from understood. We consider two problems related to the min-max risk (expected loss) of estimating an unknown $k$-state Markov chain from its $n$ sequential samples: \emph{predicting} the conditional distribution of the next sample with respect to the KL-divergence, and \emph{estimating} the transition matrix with respect to a natural loss induced by KL or a more general $f$-divergence measure.

For the first measure, we determine the min-max prediction risk to within a linear factor in the alphabet size, showing it is $\Omega(k\log\log n/n)$ and $\mathcal{O}(k^2\log\log n/n)$. For the second, if the transition probabilities can be arbitrarily small, then only trivial uniform risk upper bounds can be derived. We therefore consider transition probabilities that are bounded away from zero, and resolve the problem for essentially all sufficiently smooth $f$-divergences, including KL-, $L_2$-, Chi-squared, Hellinger, and Alpha-divergences.
\end{abstract}
\section{Introduction}\label{sec1}

Many natural phenomena are inherently probabilistic.
With past observations at hand, probabilistic models can therefore
help us predict, estimate, and understand, future outcomes and trends. 
The two most fundamental probabilistic models for sequential data
are \iid processes and Markov chains. In an \iid\ process, for
each $i\geq{1}$, a sample $X_i$ is generated independently
according to the same underlying distribution.
In Markov chains, for each $i\ge2$, the distribution of sample $X_i$
is determined by just the value of $X_{i-1}$.

Let us confine our discussion to random processes over finite
alphabets, without loss of generality, assumed to be 
$[k]:=\{1,2,\ldots,k\}$. 
An \iid\ process is defined by a single distribution $p$ over $[k]$,
while a Markov chain is characterized by a transition probability matrix $M$ over $[k]\times[k]$.
We denote the initial and stationary distributions of a Markov model by $\mu$ and $\pi$, respectively. 
For notational consistency let $P=(p)$ denote an \iid model
and $P=(M)$ denote a Markov model.

Having observed a sample sequence $X^n:= X_1, \ldots, X_n$ from an
\emph{unknown} \iid process or Markov chain, a natural problem is to \emph{predict} the
next sample point $X_{n+1}$. 
Since $X_{n+1}$ is a random variable, this task is typically
interpreted as estimating the conditional probability distribution
$P_{x^n}:=\text{Pr}{(X_{n+1}=\cdot|X^n = x^n)}$ of the next sample point ${X_{n+1}}$. 

Let $[k]^*$ denote the collection of all finite-length sequences over
$[k]$. 

Therefore, conditioning on $X^n=x^n$, our first objective
is to estimate the conditional distribution
To be more precise, we would like to find an
\emph{estimator} $\hat{P}$, that associates with every sequence
$x^n\in [k]^*$ a distribution $\hat{P}_{x^n}$ over $[k]$ that
approximates $P_{x^n}$ in a suitable sense.

Perhaps a more classical problem is \emph{parameter estimation},
which describes the underlying process. An \iid\ process is 
completely characterized by $P_{x^n}=p$, hence this problem coincides
with the previous one. For Markov chains, we seek to estimate the
transition matrix $M$. Therefore, instead of producing a probability
distribution $\hat P_{x^n}$, the estimator $\hat{M}$ maps every sequence $x^n\in
[k]^*$ to a transition matrix $\hat M_{x^n}$ over $[k]\times[k]$.  

For two distributions $p$ and $q$ over $[k]$, let $L(p,q)$ be
the \emph{loss} when $p$ is approximated by $q$.
For the prediction problem, we measure the performance of an estimator
$\hat P$ in terms of its
\emph{prediction risk}, 
\[
\rho_{n}^{L}(P,\hat{P})
:=
\Exp_{X^n\sim P}[L(P_{X^n},\hat{P}_{X^n})]
=\sum_{x^n\in [k]^n} P(x^n)L(P_{x^n},\hat{P}_{x^n}),
\]
the expected loss with respect to the sample
sequence $X^n$, where $P(x^n):=\text{Pr}{(X^n=x^n)}$.

For the estimation problem, we quantify the performance of the
estimator by \emph{estimation risk}. We first consider the expected
loss of $\hat M$ with respect to a single state $i\in [k]$:
\[
\Exp_{X^n\sim (M)}[L(M(i,\cdot), \hat M_{X^n}(i,\cdot))].
\]
We then define the estimation risk of $\hat{M}$ given sample sequence $X^n$ as the maximum expected loss over all states,
\[
{\varepsilon}_n^L(M,\hat{M}) := \max_{i\in [k]}\Exp_{X^n\sim (M)}[L(M(i,\cdot), \hat{M}_{X^n}(i,\cdot))].
\]

While the process $P$ we are trying to learn is unknown, it often
belongs to a known collection $\mathscr{P}$. 
The worst prediction risk of an estimator $\hat{P}$ over all
distributions in $\mathscr{P}$ is
\[
\rho_n^L(\mathscr{P}, \hat{P}):= \max_{P\in \mathscr{P}} \rho_n^L(P,\hat{P}).
\]
The minimal possible worst-case prediction risk, or simply the
\emph{minimax prediction risk}, incurred by any estimator is
\[
\rho_n^L(\mathscr{P}):= \min_{\hat{P}}\rho_n^L(\mathscr{P}, \hat{P})=\min_{\hat{P}}\max_{P\in \mathscr{P}} \rho_n^L(P,\hat{P}).
\]
The \emph{worst-case estimation risk} ${\varepsilon}_n^L(\mathscr{P},\hat{M})$
and the \emph{minimax estimation risk} ${\varepsilon}_n^L(\mathscr{P})$
are defined similarly. 
Given $\mathscr{P}$, our goals are 
to approximate the minimax prediction/estimation risk to a universal
constant-factor, and to devise estimators that achieve this performance.

An \emph{alternative} definition of the estimation risk, considered in~\cite{Moein16N}
and mentioned by a reviewer, is 
\[
\tilde{{\varepsilon}}_n^L(M,\hat{M}) := \sum_{i\in [k]} \pi_i \cdot \Exp_{X^n\sim (M)}[L(M(i,\cdot), \hat{M}_{X^n}(i,\cdot))].
\]
We denote the corresponding \emph{minimax estimation risk} by $\tilde{{\varepsilon}}_n^L(\mathscr{P})$.

{\em Let $o(1)$ represent a quantity that vanishes as $n\rightarrow \infty$. In the following, we use $a\lesssim b$ to denote $a\leq b(1+o(1))$, and $a\asymp b$ to denote $a\leq b(1+o(1))$ and $b\leq a(1+o(1))$.}

For the collection $\mathds{IID}^k$ of all the \iid\ processes over
$[k]$, the above two formulations coincide and the problem is
essentially the classical discrete distribution estimation
problem. The problem of determining $\rho_n^L(\mathds{IID}^k)$ was
introduced by~\cite{gilbert}
and studied in a sequence of
papers~\cite{cov72,kri81,Braess02,Pan04,BraessS04}.
For fixed $k$ and KL-divergence loss, as $n$ goes to
infinity, ~\cite{BraessS04} showed that
\[
\rho_n^{\text{KL}}(\mathds{IID}^k)
\asymp
\frac{k-1}{2n}.
\]
KL-divergence and many other important similarity measures between
two distributions can be expressed as
$f$-divergences~\cite{Csi67}. Let $f$ be a convex function with
$f(1)=0$, the $f$-divergence between two distributions $p$ and $q$
over $[k]$, whenever well-defined, is $D_f(p,
q):=\sum_{i\in[k]}q(i)f\Paren{{p(i)}/{q(i)}}$.  Call an $f$-divergence
\emph{ordinary} if $f$ is thrice continuously differentiable over
$(0,\infty)$, sub-exponential, namely,
$\lim_{x\rightarrow\infty}|f(x)|/e^{cx}=0$ for all $c>0$, and
satisfies $f''(1)\ne 0$.

Observe that all the following notable measures are ordinary $f$-divergences:
Chi-squared divergence~\cite{Fran14} from $f(x)=(x-1)^2$,
KL-divergence~\cite{Solo51} from $f(x)=x\log x$, Hellinger
divergence~\cite{Mik01} from $f(x)=(\sqrt{x}-1)^2$, and
Alpha-divergence~\cite{Gav17} from
$f_{\alpha}(x):={4}(1-x^{(1+\alpha)/2})/{(1-\alpha^2)}$,  where
$\alpha\not=\pm 1$.
\paragraph{Related Work} For any $f$-divergence, we denote the
corresponding minimax prediction risk for an $n$-element sample over
set $\mathscr{P}$ by $\rho^f_n(\mathscr{P})$.
Researchers in~\cite{KamathOPS15} considered the problem of determining $\rho^f_n(\mathds{IID}^k)$ for the ordinary $f$-divergences.
Except the above minimax formulation, recently, researchers also considered formulations that are more adaptive to the underlying \iid\ processes~\cite{gt15}~\cite{va16}.
Surprisingly, while the min-max risk of \iid\ processes was addressed in a large
body of work, the risk of Markov chains, which frequently arise in practice, was
not studied until very recently.

Let $\mathds{M}^k$ denote the collection of all the Markov chains over $[k]$.
For prediction with KL-divergence,~\cite{Moein16} showed that
$\rho_n^{\text{KL}}(\mathds{M}^k)=\Theta_k\left({\log\log n}/{n}\right)$,
but did not specify the dependence on $k$.
For estimation,~\cite{Geoffrey18} considered 
the class of Markov Chains whose pseudo-spectral gap is bounded away from
0 and approximated the $L_1$ estimation risk to within a $\log n$
factor. Some of their techniques, in particular the
lower-bound construction in their displayed equation $(4.3)$, are 
of similar nature and were derived independently of results in
Section~\ref{est_low} in our paper. 

Our first main result determines the dependence of
$\rho_n^{\text{KL}}(\mathds{M}^k)$ on both $k$ and $n$,
to within a factor of roughly $k$:
\begin{theorem}
The minimax KL-prediction risk of Markov chains satisfies
\[
\frac{(k-1)\log\log n}{4en}\lesssim \rho_n^{\text{KL}}(\mathds{M}^k)\lesssim\frac{2k^2\log\log n}{n}.
\]
\end{theorem}

Depending on $M$, some states may be observed very
infrequently, or not at all. This does not drastically affect 
the prediction problem as these states will be also have small
impact on $\rho_n^{\text{KL}}(\mathds{M}^k)$ in the prediction risk $\rho_{n}^{L}(P,\hat{P})$. 
For estimation, however, rare and unobserved states still
need to be well approximated, hence
$\varepsilon_n^L(\mathds{M}^k)$
does not decrease with $n$, and for example
$\varepsilon_n^{\text{KL}}(\mathds{M}^k)=\log k$ for all $n$.

We therefore parametrize the risk by the lowest probability
in the transition matrix.
For $\delta>0$ let 
\[
\mathds{M}^k_{\delta}
:=\{(M): M_{i,j}\ge\delta,\ \forall i,j\},
\]
be the collection of Markov chains whose lowest transition
probability exceeds $\delta$. Note that $\mathds{M}^k_{\delta}$ is trivial if $\delta\ge 1/k$, we only consider $\delta\in(0,1/k)$.
We characterize the minimax estimation risk of $\mathds{M}^k_{\delta}$
almost precisely. 
\begin{theorem}\label{thm2}
For all ordinary $f$-divergences and all $\delta\in(0,1/k)$,
\[
\tilde{{\varepsilon}}_n^f(\mathds{M}^k_{\delta})\asymp\frac{(k-1)kf''(1)}{2n}
\]
and 
\[
(1-\delta)\frac{(k-2)f''(1)}{2n\delta}\lesssim \varepsilon_n^{\text{f}}(\mathds{M}^k_{\delta})\lesssim\frac{(k-1)f''(1)}{2n\delta}.
\]
\end{theorem}

We can further refine the estimation-risk bounds by partitioning $\mathds{M}^k_{\delta}$
based on the smallest probability in the chain's stationary distribution $\pi$.
Clearly, $\min_{i\in[k]}\pi_i\le 1/k$. 
For $0<\pi^*\le 1/k$ and $0<\delta<1/k$, let
\[
\mathds{M}^k_{\delta,\pi^*}
:=\{(M): (M)\in \mathds{M}^k_{\delta} \text{ and } 
\min_{i\in[k]}\pi_i=\pi^*\}
\]
be the collection of all Markov chains in $ \mathds{M}^k_{\delta}$
whose lowest stationary probability is $\pi^*$.  
We determine the minimax estimation risk over
$\mathds{M}^k_{\delta,\pi^*}$ nearly precisely. 

\begin{theorem}\label{thm3}
For all ordinary $f$-divergences, 
\[
(1-\pi^*)\frac{(k-2)kf''(1)}{2n}\lesssim\tilde{{\varepsilon}}_n^f(\mathds{M}^k_{\delta,\pi^*})\lesssim\frac{(k-1)kf''(1)}{2n}
\]
and 
\[
(1-\pi^*)\frac{(k-2)f''(1)}{2n\pi^*}\lesssim \varepsilon_n^{\text{f}}(\mathds{M}^k_{\delta,\pi^*})\lesssim\frac{(k-1)f''(1)}{2n\pi^*}.
\]
\end{theorem}

For $L_2$-distance corresponding to the squared Euclidean norm, we prove the following risk bounds.
\begin{theorem}
For all $\delta\in(0,1/k)$,
\[
\tilde{{\varepsilon}}_n^{L_2}(\mathds{M}^k_{\delta})\asymp\frac{k-1}{n}
\]
and
\[
(1-\delta)^2\frac{1-\frac{1}{k-1}}{n\delta}
\lesssim
\varepsilon_n^{L_2}(\mathds{M}^k_{\delta})
\lesssim
\frac{1-\frac{1}{k}}{n\delta}.
\]

\end{theorem}

\begin{theorem}
For all $\delta\in(0,1/k)$ and $\pi^*\in(0,1/k]$, 
\[
(1-\pi^*)^2\frac{k-\frac{k}{k-1}}{n}\lesssim\tilde{{\varepsilon}}_n^{L_2}(\mathds{M}^k_{\delta,\pi^*})\lesssim\frac{k-1}{n}
\]
and
\[
(1-\pi^*)^2\frac{1-\frac{1}{k-1}}{n\pi^*}
\lesssim
\varepsilon_n^{L_2}(\mathds{M}^k_{\delta,\pi^*})
\lesssim
\frac{1-\frac{1}{k}}{n\pi^*}.
\]
\end{theorem}

The rest of the paper is organized as follows.
Section~\ref{def} introduces add-constant estimators 
and additional definitions and notation for Markov chains.
Note that each of the above 
results consists of a lower bound and an upper bound.
We prove the lower bound by constructing a suitable prior distribution over the
relevant collection of processes.
Section~\ref{pred_low} and ~\ref{est_low} 
describe these prior distributions for the
prediction and estimation problems, respectively. 
The upper bounds are derived via simple variants of the standard
add-constant estimators.
Section~\ref{pred_upp} and ~\ref{est_upp} describe 
the estimators for the prediction and estimation bounds, respectively.
 For space considerations, we relegate
all the proofs to Section~\ref{sec:9} to~\ref{sec4}.

\section{Definitions and Notation}\label{def}
\subsection{Add-constant estimators}
Given a sample sequence $X^n$ from an \iid\ process $(p)$, let $N_i'$
denote the number of times symbol $i$ appears in $X^n$. The
classical \emph{empirical estimator} estimates $p$ by 

\[
\hat{p}_{X^n}(i) := \frac{N_i'}{n},\ \forall i\in[k].
\]
The empirical estimator performs poorly for loss measures such as 
KL-divergence. For example, if $p$ assigns
a tiny probability to a symbol so that it is unlikely to
appear in $X^n$, then with high probability the KL-divergence between
$p$ and $\hat{p}_{X^n}$ will be infinity.

A common solution applies the Laplace smoothing technique~\cite{Chung12}
that assigns to each symbol $i$ a probability proportional to
$N'_i+\beta$, where $\beta>0$ is a fixed constant.
The resulting \emph{add-$\beta$} estimator, is denoted by
$\hat{p}^{+\beta}$.
Due to their simplicity and
effectiveness, add-$\beta$ estimators are widely used in various
machine learning algorithms such as naive Bayes
classifiers~\cite{Bish06}. As shown in~\cite{BraessS04},
for the \iid\ processes, a variant of the add-${3}/{4}$ estimator
achieves the minimax estimation risk
$\rho_n^{\text{KL}}(\mathds{IID}^k)$.

Analogously, given a sample sequence $X^n$ generated by a Markov
chain, let $N_{ij}$ denote the number of times symbol $j$ appears
right after symbol $i$ in $X^n$, and let $N_{i}$ denote the number of
times that symbol $i$ appears in $X^{n-1}$. We define the
add-$\beta$ estimator $\hat{M}^{+\beta}$ as 
\[
\hat{M}^{+\beta}_{X^n}(i,j) := \frac{N_{ij}+{\beta}}{N_{i}+k\beta},\ \forall i,j\in[k].
\]

\subsection{More on Markov chains}
Adopting notation in~\cite{Yuv17}, let $\Delta_k$ denote the
collection of discrete distributions over $[k]$. Let $[k]^e$ and
$[k]^o$ be the collection of even and odd integers in $[k]$,
respectively. By convention, for a Markov chain over $[k]$, we call
each symbol $i\in[k]$ a \emph{state}. 
Given a Markov chain, the \emph{hitting time} $\tau(j)$ is the
first time the chain reaches state $j$.
We denote by $\text{Pr}_i({\tau(j)=t})$
the probability that starting from $i$, the hitting time of $j$ is
exactly $t$. For a Markov chain $(M)$, we denote by $P^t$ the
distribution of $X_t$ if we draw $X^t\sim (M)$.  Additionally,
for a fixed Markov chain $(M)$, the \emph{mixing time} $t_{mix}$ denotes the
smallest index $t$ such that ${L_1}(P^{t},\pi)<{1}/{2}$.  Finally, for
notational convenience, we write $M_{ij}$ instead of $M(i,j)$ whenever
appropriate.

\section{Minimax prediction: lower bound}\label{pred_low}
A standard lower-bound argument for minimax prediction risk uses the
fact that
\[
\rho_n^{\text{KL}}(\mathscr{P})
=
\min_{\hat{P}}\max_{P\in \mathscr{P}} \rho_n^{\text{KL}}(P,\hat{P})
\geq
\min_{\hat{P}}\Exp_{P\sim\Pi} [\rho_n^{\text{KL}}(P,\hat{P})]
\]
for any prior distribution $\Pi$ over $\mathscr{P}$. 
One advantage of this approach is that 
the optimal estimator that minimizes
$\Exp_{P\sim\Pi} [\rho_n^{\text{KL}}(P,\hat{P})]$ can often be computed
explicitly. 

Perhaps the simplest prior is the uniform distribution
$U(\mathscr{P}_S)$ over a subset $\mathscr{P}_S\subset\mathscr{P}$. Let
$\hat{P}^*$ be the optimal estimator minimizing $\Exp_{P\sim
  U(\mathscr{P}_S)} [\rho_n^{\text{KL}}(P,\hat{P})]$. 
Computing $\hat{P}^*$ for all the possible sample sequences $x^n$ may
be unrealistic. Instead, let ${\mathscr{K}}_n$ be an arbitrary subset
of $[k]^n$, we can lower bound 

\[
\rho_n^{\text{KL}}(P,\hat{P}) = \Exp_{X^n\sim P}[D_{\text{KL}}(P_{X^n},\hat{P}_{X^n})]
\]
by
\[
\rho_n^{\text{KL}}(P,\hat{P}; \mathscr{K}_n) := \Exp_{X^n\sim P}[D_{\text{KL}}(P_{X^n},\hat{P}_{X^n})\indic_{X^n\in \mathscr{K}_n}].
\]
Hence,
\[
\rho_n^{\text{KL}}(\mathscr{P})\geq{ \min_{\hat{P}}\Exp_{P\sim U(\mathscr{P}_S)} [\rho_n^\text{KL}(P,\hat{P}; \mathscr{K}_n)]}.
\]
The key to applying the above arguments is to find a proper pair $(\mathscr{P}_S, \mathscr{K}_n)$.

Without loss of generality, assume that $k$ is even.
Let $a:=\frac{1}{n}$ and $b:=1-\frac{k-2}{n}$, and define
\newcommand{\fon}{a}
\newcommand{\omf}{b-a}
\newcommand{\omfs}[1]{
b\!\!-\!\!p_{#1}
}
\[
M_n(p_2, p_4, \ldots, p_k) := 
\begin{bmatrix}
    \omf & \fon & \fon &  \fon & \dots  & \fon & \fon\\
    p_2 & \omfs2 & \fon & \fon & \dots  & \fon & \fon\\
    \fon & \fon & \omf & \fon & \dots  & \fon & \fon\\
    \fon & \fon & p_4 & \omfs4 &\dots  & \fon & \fon\\
    \vdots & \vdots & \vdots & \vdots & \ddots & \vdots &\vdots\\
    \fon & \fon & \fon & \fon & \dots  & \omf & \fon\\
    \fon & \fon & \fon & \fon & \dots  & p_k & \omfs k
\end{bmatrix}.
\]
In addition, let 
\[
V_n:= \left\{\frac{1}{\log^t n}: t\in \mathds{N} \text{ and } 1\leq t \leq \frac{\log n}{2\log\log n} \right\},
\]
and let $u_k$ denote the uniform distribution over $[k]$.
Finally, given $n$, define
\[
\mathscr{P}_S=\{(M)\in  \mathds{M}^k: \mu=u_{k} \text{ and }M=M_n(p_2, p_4, \ldots, p_k), \text{ where } p_i\in V_n, \forall i\in[k]^e\}.
\]
Next, let $\mathscr{K}_n$ be the collection of sequences
$x^n\in[k]^n$ whose last appearing state
didn't transition to any other state.
For example, 3132, or 31322, but not 21323.
In other words, for any state $i\in[k]$, let $\bar{i}$
represent an arbitrary state in $[k]\setminus \{i\}$, then

\[
\mathscr{K}_n=\{x^n\in[k]^n: x^n={\bar{i}}^{n-\ell} i^\ell: i\in[k], n-1\geq \ell\geq{1}\}.
\]
\section{Minimax prediction: upper bound}\label{pred_upp}
For the $\mathscr{K}_n$ defined in the last section, 
\[
\rho_n^{\text{KL}}(P,\hat{P}; \mathscr{K}_n) = \sum_{x^n\in \mathscr{K}_n}P(x^n)D_{\text{KL}}(P_{x^n},\hat{P}_{x^n}).
\]
We denote the \emph{partial minimax prediction risk} over $\mathscr{K}_n$ by
\[
\rho_n^{\text{KL}}(\mathscr{P};  \mathscr{K}_n):=\min_{\hat{P}}\max_{P\in \mathscr{P}}\rho_n^{\text{KL}}(P,\hat{P}; \mathscr{K}_n).
\]
Let $\overline{\mathscr{K}_n}:=[k]^n\setminus\mathscr{K}_n$.
Define $\rho_n^{\text{KL}}(P,\hat{P}; \overline{\mathscr{K}_n})$ and $\rho_n^{\text{KL}}(\mathscr{P};  \overline{\mathscr{K}_n})$ in the same manner. As the consequence of $\hat{P}$ being a function from $[k]^n$ to $\Delta_k$, we have the following triangle inequality,
\[
\rho_n^{\text{KL}}(\mathscr{P}) \leq  \rho_n^{\text{KL}}(\mathscr{P}; \overline{\mathscr{K}_n}) + \rho_n^{\text{KL}}(\mathscr{P}; \mathscr{K}_n).
\]
Turning back to Markov chains, let $\hat{P}^{+\frac{1}{2}}$ denote the estimator that maps $X^n\sim (M)$ to $\hat{M}^{+\frac{1}{2}}(X_{n}, \cdot)$, one can show that
\[
\rho_n^{\text{KL}}(\mathds{M}^k; \overline{\mathscr{K}_n})\leq\max_{P\in \mathds{M}^k}\rho_n^{\text{KL}}(P, \hat{P}^{+\frac{1}{2}};  \overline{\mathscr{K}_n})\leq{\mathcal{O}_k\left(\textstyle{\frac{1}{n}}\right)}.
\]

Recall the following lower bound
\[
\rho_n^{\text{KL}}(\mathds{M}^k)=\Omega_k\Paren{\frac{\log\log n}{n}}.
\]
This together with the above upper bound on $\rho_n^{\text{KL}}(\mathds{M}^k; \overline{\mathscr{K}_n})$ and the triangle inequality shows that an upper bound on $\rho_n^{\text{KL}}(\mathds{M}^k; \mathscr{K}_n)$ also suffices to bound the leading term of $\rho_n^{\text{KL}}(\mathds{M}^k)$. The following construction yields such an upper bound. 

We partition $\mathscr{K}_n$ according to the last appearing state and
the number of times it transitions to itself, 
\[
\mathscr{K}_n=\cup_{\ell=1}^{n-1}K_\ell(i), \text{ where } K_\ell(i):= \{x^n\in[k]^n: x^n={\bar{i}}^{n-\ell} i^\ell\}.
\]
For any $x^n\in \mathscr{K}_n$, there is a unique $K_\ell(i)$ such that $x^n\in K_\ell(i)$. Consider the following estimator 
\[ 
\hat{P}_{x^n}(i)
:=
\begin{cases}  
      1-\frac{1}{\ell\log n}& \ell\leq\frac{n}{2}\\ 
      1-\frac{1}{\ell}&  \ell>\frac{n}{2}
   \end{cases}
\]
and 
\[ 
\hat{P}_{x^n}(j)
:=\frac{1-\hat{P}_{x^n}(i)}{k-1},\ \forall j\in[k]\setminus\{i\},
\]
we can show that
\[
\rho_n^{\text{KL}}(\mathds{M}^k;  \mathscr{K}_n)\leq \max_{P\in \mathds{M}^k}\rho_n^{\text{KL}}(P,\hat{P}; \mathscr{K}_n)\lesssim\frac{2k^2\log\log n}{n}.
\]
The upper-bound proof applies the following lemma that
uniformly bounds the hitting probability of any $k$-state Markov chain.
\begin{Lemma}~\cite{sup17}\label{lemma6}
For any Markov chain over $[k]$ and any two states $i,j\in[k]$, if $n>k$, then
\[
\textrm{Pr}_i{(\tau(j)=n)}\leq \frac{k}{n}.
\]
\end{Lemma}

\section{Minimax estimation: lower bound}\label{est_low}
Analogous to Section~\ref{pred_low}, we use the following standard argument to lower bound the minimax risk
\[
\varepsilon_n^L(\mathscr{M})=\min_{\hat{M}}\max_{(M)\in \mathscr{M}} \varepsilon_n^L(M,\hat{M})\geq{ \min_{\hat{M}}\Exp_{(M)\sim U(\mathscr{M}_S)} [\varepsilon_n^L(M,\hat{M})]},
\]
where $\mathscr{M}_S\subset \mathscr{M}$ and $U(\mathscr{M}_S)$ is the uniform distribution over $\mathscr{M}_S$. Setting $\mathscr{M}=\mathds{M}^k(\delta, \pi^*)$, we outline the construction of $\mathscr{M}_S$ as follows.

Let $u_{k-1}$ be the uniform distribution over $[k-1]$. 
As in~\cite{KamathOPS15}, denote the $L_{\infty}$ ball of radius $r$
around $u_{k-1}$ by
\[
B_{k-1}(r):=\{p\in \Delta_{k-1}: {L_\infty}(p, u_{k-1})<r \},
\]
where ${L_\infty}(\cdot,\cdot)$ is the $L_\infty$ distance between two distributions.
Define
\newcommand{\nor}{\frac{\bar{\pi}^*}{k-1}}
\[
p':=(p_1,\ p_2,\ \ldots,\ p_{k-1}),
\]
\[
p^*:= \Paren{\nor,\ \nor,\  \dots\ \nor,\ \pi^*},
\]
and
\[
M_n(p') 
:=
\begin{bmatrix}
    \nor & \nor  &  \dots  & \nor  & \pi^*\\
    \nor & \nor  &  \dots  & \nor & \pi^*\\
    \vdots & \vdots  & \ddots & \vdots &\vdots\\
    \nor & \nor   & \dots  & \nor & \pi^*\\
   \bar{\pi}^*p_1 &  \bar{\pi}^*p_2 & \dots  & \bar{\pi}^*p_{k-1} & \pi^*
\end{bmatrix},
\]
where $\bar{\pi}^*=1-\pi^*$ and $\sum_{i=1}^{k-1} p_i=1$.

Given $n$ and $\epsilon\in(0,1)$, let $n':=(n(1+\epsilon)\pi^*)^{1/5}$. We set 
\[
\mathscr{M}_S =\{(M)\in \mathds{M}^k(\delta, \pi^*): \mu= p^* \text{ and }M=M_n({p'}), \text{ where } p'\in B_{k-1}(1/n')\}.
\]
Noting that the uniform distribution over $\mathscr{M}_S$, $U(\mathscr{M}_S)$, is induced by $U(B_{k-1}(1/n'))$, the uniform distribution over $B_{k-1}(1/n')$ and thus is well-defined.

An important property of the above construction is that for a sample sequence $X^n\sim (M)\in \mathscr{M}_S$, $N_k$,  the number of times that state $k$ appears in $X^n$, is a binomial random variable with parameters $n$ and $\pi^*$. Therefore, by the following lemma, $N_k$ is highly concentrated around its mean $n\pi^*$. 
\begin{Lemma}~\cite{com06}\label{lemma9}
Let $Y$ be a binomial random variable with parameters $m\in\mathds{N}$ and $p\in[0,1]$, then for any $\epsilon\in(0,1)$,
\[
\textrm{Pr}{(Y\geq(1+\epsilon) mp)}\leq \exp\left(-\epsilon^2mp/3\right).
\]
\end{Lemma}

In order to prove the lower bound on $\tilde{{\varepsilon}}_n^f(\mathds{M}^k_{\delta,\pi^*})$, we only need to modify the above  construction as follows. Instead of drawing the last row of the transition matrix $M_n(p')$ uniformly from the distribution induced by $U(B_{k-1}(1/n'))$, we draw all rows independently in the same fashion. The proof is omitted due to similarity.

\section{Minimax estimation: upper bound}\label{est_upp}
The proof of the upper bound relies on a concentration inequality for Markov chains in $\mathds{M}^k_{\delta}$, which can be informally expressed as
\[
\textrm{Pr}(|N_i-(n-1)\pi(i)|>t)\leq \Theta_\delta(\exp(\Theta_\delta(-t^2/n))).
\]
Note that this inequality is very similar to the Hoeffding's inequality for \iid\ processes. 

The difficulty in analyzing the performance of the original
add-$\beta$ estimator is that the chain's initial distribution could
be far away from its stationary distribution and finding a simple
expression for $\Exp[N_i]$ and $\Exp[N_{ij}]$ could be hard. To overcome
this difficulty, we ignore the first few sample points and construct a
new add-$\beta$ estimator based on the remaining sample points. 
Specifically, let $X^n$ be a length-$n$ sample sequence drawn from
the Markov chain $(M)$. Removing the first $m$ sample points,
$X^n_{m+1}:=X_{m+1},\ldots, X_n$ can be viewed as a
length-$(n\!\!-\!\!m)$ sample sequence drawn from $(M)$ whose initial
distribution $\mu'$ satisfies
\[
{L_1}(\mu', \pi)<2(1-\delta)^{m-1}.
\]
Let $m=\sqrt{n}$. For sufficiently large $n$, ${L_1}(\mu', \pi)\ll
1/n^2$ and $\sqrt{n}\ll n$. 
Hence without loss of generality, we assume that the original initial
distribution $\mu$ already satisfies ${L_1}(\mu, \pi)<1/n^2$.
If not, we can simply replace $X^n$ by $X^n_{\sqrt{n}+1}$.

To prove the desired upper bound for ordinary $f$-divergences, it suffices to use the add-$\beta$ estimator 
\[
\hat{M}^{+\beta}_{X^n}(i,j)
:=
\frac{N_{ij}+{\beta}}{N_{i}+k\beta},\ \forall i,j\in[k].
\]
For the $L_2$-distance, instead of an add-constant estimator, we apply an add-${\sqrt{N_{i}}}/{k}$ estimator
\[
\hat{M}^{+{\sqrt{N_{i}}}/{k}}_{X^n}(i,j) := \frac{N_{ij}+{\sqrt{N_{i}}}/{k}}{N_{i}+{\sqrt{N_{i}}}},\ \forall i,j\in[k].
\]

\section{Experiments}
\label{exp}
We augment the theory with experiments 
that demonstrate the efficacy of our proposed estimators
and validate the functional form of the derived bounds.

We briefly describe the experimental setup.
For the first three figures, $k=6$, $\delta=0.05$, and $10,000\leq n\leq 100,000$. 
For the last figure, $\delta=0.01$, $n=100,000$, and $4\le k \le 36$.
In all the experiments, the initial distribution $\mu$ of the Markov chain is drawn from the
$k$-Dirichlet($1$) distribution. For the transition matrix $M$, we
first construct a transition matrix $M'$ where each row is 
drawn independently from the $k$-Dirichlet($1$) distribution.
To ensure that each element of $M$ is at least $\delta$,
let ${J}_k$ represent the $k\times k$ all-ones matrix, 
and set $M=M'(1-k\delta)+\delta {J}_k$. 
We generate a new Markov chain for each curve in
the plots. And each data point on the curve shows the average loss of
$100$ independent restarts of the same Markov chain. 

The plots use the following abbreviations:
Theo for theoretical minimax-risk values;
Real for real experimental results: 
using the estimators described in Sections~\ref{pred_upp} and~\ref{est_upp};
Pre for average prediction loss and Est for average estimation loss; 
Const for add-constant estimator;
Prop for proposed add-${\sqrt{N_{i}}}/{k}$ estimator described in
Section~\ref{est_upp};
Hell, Chi, and Alpha(c) for Hellinger divergence, Chi-squared divergence, 
and Alpha-divergence with parameter $c$.
In all three graphs, the theoretical min-max curves are precisely the upper
bounds in the corresponding theorems, except that in the prediction curve in
Figure~\ref{fig:sub1} the constant factor 2 in the upper bound is
adjusted to $1/2$ to better fit the experiments.
Note the excellent fit between the theoretical bounds and experimental
results. 

Figure~\ref{fig:sub1} shows the decay of the experimental and
theoretical KL-prediction and KL-estimation losses with
the sample size $n$. 
Figure~\ref{fig:sub2} compares the $L_2$-estimation losses of our proposed
estimator and the add-one estimator, and the theoretical minimax
values. Figure~\ref{fig:sub3} compares the experimental
estimation losses and the theoretical minimax-risk values for
different loss measures. Finally, figure~\ref{fig:sub4}  presents an 
experiment on KL-learning losses that scales $k$ up while $n$ is fixed.
All the four plots demonstrate that our
theoretical results are accurate and can be used to estimate
the loss incurred in learning Markov chains. Additionally,
Figure~\ref{fig:sub2} shows that
our proposed add-${\sqrt{N_{i}}}/{k}$ estimator is uniformly better
than the traditional add-one estimator for different values of sample
size $n$. We have also considered add-constant estimators with different
constants varying from $2$ to $10$ and our proposed estimator
outperformed all of them.

\vspace{-0.08em}
\begin{figure}[h]
\centering
\begin{subfigure}{.5\textwidth}
  \centering
  \captionsetup{width=1\linewidth}
  \includegraphics[width=\linewidth]{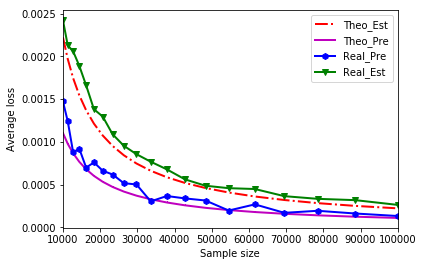}
  \caption{KL-prediction and estimation losses}
  \label{fig:sub1}
\end{subfigure}%
\begin{subfigure}{.5\textwidth}
  \centering
  \captionsetup{width=1\linewidth}
  \includegraphics[width=\linewidth]{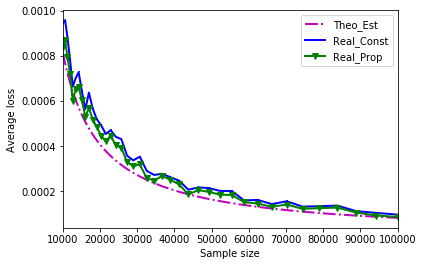}
  \caption{$L_2$-estimation losses for different estimators}
  \label{fig:sub2}
\end{subfigure}
\begin{subfigure}{.5\textwidth}
  \centering
  \captionsetup{width=1\linewidth}
  \includegraphics[width=\linewidth]{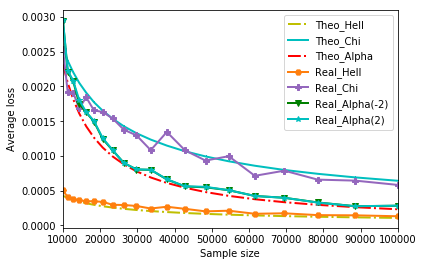}
  \caption{Hellinger, Chi-squared, and Alpha- estimation losses}
  \label{fig:sub3}
\end{subfigure}%
\begin{subfigure}{.5\textwidth}
  \centering
  \captionsetup{width=1\linewidth}
  \includegraphics[width=\linewidth]{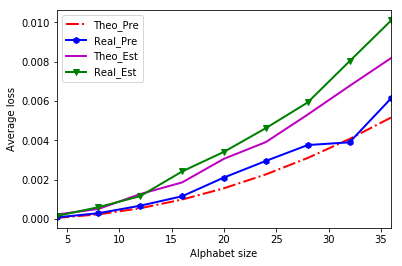}
  \caption{Fixed $n$ and varying $k$}
  \label{fig:sub4}
\end{subfigure}%
\caption{Experiments}
\label{fig:test}
\end{figure}
\vspace{-0.15em}

\section{Conclusions}
We studied the problem of learning an unknown $k$-state Markov chain
from its $n$ sequential sample points.  We considered two
formulations: prediction and estimation. 
For prediction, we determined the minimax risk up to a multiplicative factor
of $k$. 
For estimation, when the transition probabilities are bounded away from
zero, we obtained nearly matching lower and upper bounds on
the minimax risk for $L_2$ and ordinary $f$-divergences. 
The effectiveness of our proposed estimators
was verified through experimental simulations.
Future directions include
closing the gap in the prediction problem in Section~\ref{sec1},
extending the results on the min-max
estimation problem to other classes of Markov chains,
and extending the work from the classical setting $k\ll n$,
to general $k$ and $n$. 
\vfill
\pagebreak


%

\medskip


\vfill
\pagebreak

\section{Minimax prediction: lower bound}\label{sec:9}
A standard argument for lower bounding the minimax prediction risk is 
\[
\rho_n^{\text{KL}}(\mathscr{P})=\min_{\hat{P}}\max_{P\in \mathscr{P}} \rho_n^{\text{KL}}(P,\hat{P})\geq{ \min_{\hat{P}}\EE_{P\sim\Pi} [\rho_n^{\text{KL}}(P,\hat{P})]},
\]
where $\Pi$ is a prior distribution over $\mathscr{P}$. The advantage of this approach is that the optimal estimator that minimizes $\EE_{P\sim\Pi} [\rho_n^{\text{KL}}(P,\hat{P})]$ can often be computed explicitly. 

Perhaps the simplest prior is the uniform distribution over some subset of $\mathscr{P}$. Consider the uniform distribution over $\mathscr{P}_S\subset\mathscr{P}$, say $U(\mathscr{P}_S)$, the following lemma shows an explicit way of computing the optimal estimator for $\EE_{P\sim U(\mathscr{P}_S)} [\rho_n^{\text{KL}}(P,\hat{P})]$ when $\mathscr{P}_S$ is finite.
\begin{Lemma}\label{lemma1}
Let $\hat{P}^*$ be the optimal estimator that minimizes $\EE_{P\sim U(\mathscr{P}_S)} [\rho_n^{\text{KL}}(P,\hat{P})]$, then for any $x^n\in[k]^n$ and any symbol $i\in[k]$,
\[
\hat{P}^*_{x^n}(i)=\sum_{P\in\mathscr{P}_S}\frac{P({x^n})}{\sum_{P'\in\mathscr{P}_S}P'({x^n})} P_{x^n}(i).
\]
\end{Lemma}
Clearly, computing $\hat{P}^*$ for all the possible sample sequences $x^n$ may be unrealistic. Instead, let ${\mathscr{K}}_n$ be an arbitrary subset of $[k]^n$, we can lower bound 
\[
\rho_n^{\text{KL}}(P,\hat{P}) = \EE_{X^n\sim P}[D_{\text{KL}}(P_{X^n},\hat{P}_{X^n})]
\]
by
\[
\rho_n^{\text{KL}}(P,\hat{P}; \mathscr{K}_n) := \EE_{X^n\sim P}[D_{\text{KL}}(P_{X^n},\hat{P}_{X^n})\indic_{X^n\in \mathscr{K}_n}].
\]
This yields
\[
\rho_n^{\text{KL}}(\mathscr{P})\geq{ \min_{\hat{P}}\EE_{P\sim U(\mathscr{P}_S)} [\rho_n^\text{KL}(P,\hat{P}; \mathscr{K}_n)]}.
\]
The key to apply the above arguments is to find a proper pair $(\mathscr{P}_S, \mathscr{K}_n)$. The rest of this section is organized as follows. In Subsection~\ref{prior}, we present our construction of $\mathscr{P}_S$ and $\mathscr{K}_n$. In Subsection~\ref{optest}, we find the exact form of the optimal estimator using Lemma~\ref{lemma1}. Then we analyze its prediction risk over $\mathscr{K}_n$ in Subsection~\ref{analysis}, where we further partition $\mathscr{K}_n$ into smaller subsets $K_\ell(i)$, and lower bound the KL-divergence over $K_\ell(i)$ and the probability $P(X^n\in K_\ell(i))$ in Lemma~\ref{lemma4} and~\ref{lemma5}, respectively. Finally, we consolidate all the previous results and prove the desired lower bound on $\rho_n^{\text{KL}}(\mathscr{P})$.
\subsection{Prior construction}\label{prior}
Without loss of generality, we assume that $k$ is an even integer. For notational convenience, we denote by $u_k$ the uniform distribution over $[k]$ and define
\[
M_n(p_2, p_4, \ldots, p_k) := 
\begin{bmatrix}
    \omf & \fon & \fon &  \fon & \dots  & \fon & \fon\\
    p_2 & \omfs2 & \fon & \fon & \dots  & \fon & \fon\\
    \fon & \fon & \omf & \fon & \dots  & \fon & \fon\\
    \fon & \fon & p_4 & \omfs4 &\dots  & \fon & \fon\\
    \vdots & \vdots & \vdots & \vdots & \ddots & \vdots &\vdots\\
    \fon & \fon & \fon & \fon & \dots  & \omf & \fon\\
    \fon & \fon & \fon & \fon & \dots  & p_k & \omfs k
\end{bmatrix},
\]
where $a:=\frac{1}{n}$ and $b:=1-\frac{k-2}{n}$. In addition, let 
\[
V_n:= \left\{\frac{1}{\log^t n}: t\in \mathds{N} \text{ and } 1\leq t \leq \frac{\log n}{2\log\log n} \right\}.
\]
Given $n$, we set 
\[
\mathscr{P}_S=\{(M)\in {\mathds{M}}^k: \mu=u_{k} \text{ and }M=M_n(p_2, p_4, \ldots, p_k), \text{ where } p_i\in V_n, \forall i\in[k]^e\}.
\]
Then, we choose $\mathscr{K}_n$ to be the collection of sequences $x^n\in[k]^n$ whose last appearing state didn't transition to any other symbol. In other words, for any state $i\in[k]$, let $\bar{i}$ represent an arbitrary state other than $i$, then
\[
\mathscr{K}_n=\{x^n\in[k]^n: x^n={\bar{i}}^{n-\ell} i^\ell: i\in[k], n-1\geq \ell\geq{1}\}.
\]
According to both the last appearing state and the number of times it transitions to itself, we can partition $\mathscr{K}_n$ as
\[
\mathscr{K}_n=\cup_{\ell=1}^{n-1}K_\ell(i), \text{ where } K_\ell(i):= \{x^n\in[k]^n: x^n={\bar{i}}^{n-\ell} i^\ell\}.
\]

\subsection{The optimal estimator}\label{optest}
Let $\hat{P}^*$ denote the optimal estimator that minimizes $\EE_{P\sim U(\mathscr{P}_S)} [\rho_n^\text{KL}(P,\hat{P}; \mathscr{K}_n)]$. The following lemma presents the exact form of $\hat{P}^*$.
\begin{Lemma}\label{lemma2}
For any $x^n\in \mathscr{K}_n$, there exists a unique $K_\ell(i)$ that contains it. Consider $\hat{P}^*_{x^n}$, we have:

\begin{enumerate}
\item If $i\in[k]^e$, then
\[ 
\hat{P}^*_{x^n}(j)
:=
\begin{cases}  
      a & j>i \text{ or } j<i-1\\ 
      \sum_{v\in V_n}(b-v)^\ell/\sum_{v\in V_n}(b-v)^{\ell-1} & j=i\\
      \sum_{v\in V_n}(b-v)^{\ell-1}v/\sum_{v\in V_n}(b-v)^{\ell-1} & j=i-1
   \end{cases}
\]
\item If $i\in[k]^o$, then
\[ 
\hat{P}^*_{x^n}(j)
:=
\begin{cases}  
      a & j>i \text{ or } j<i\\ 
      b-a & j=i
   \end{cases}
\]
\end{enumerate}
\end{Lemma}

\begin{proof}
Given $(M)\in \mathscr{P}_S$, consider $X^n\sim (M)$, 
\begin{align*}
\Pr(X^n=x^n)&= \frac{1}{k}{\displaystyle \prod_{i_1\in[k]} {\prod_{j_1\in[k]}} M_{i_1j_1}^{N_{i_1j_1}}}.
\end{align*}

By Lemma~\ref{lemma1}, for any $x^n\in{K_\ell(i)}$ and $j\in[k]$, $\hat{P}^*_{x^n}(j)$ evaluates to
\begin{align*}
\hat{P}^*_{x^n}(j)
&=\frac{\displaystyle\sum_{(M)\in \mathscr{P}_S}{M_{ij}}\prod_{i_1\in[k]} {\prod_{j_1\in[k]}} M_{i_1j_1}^{N_{i_1j_1}}}{\displaystyle\sum_{(M)\in \mathscr{P}_S}\prod_{i_1\in[k]} {\prod_{j_1\in[k]}} M_{i_1j_1}^{N_{i_1j_1}}}.
\end{align*}

Noting that $x^n\in{K_\ell(i)}$ implies $N_{ii}=\ell-1$ and $N_{ij}=0, \forall j\not=i$. Besides, for any $j_1\in[k]$ and $i_1\in [k]\setminus{\{j_1,j_1+1\}}$, $M_{i_1j_1}$ is uniquely determined by $i_1$ and $j_1$ for all $(M)\in \mathscr{P}_S$.

Thus, for $s=0$ or $1$, we can rewrite ${M_{ij}}^s\prod_{i_1\in[k]} {\prod_{j_1\in[k]}} M_{i_1j_1}^{N_{i_1j_1}}$ as
\begin{center}
$C(x^n,k){\displaystyle {M_{ij}^s}\prod_{\substack{t=2\\ t\ even}}^{k}} \left[M_{t(t-1)}\right]^{N_{t(t-1)}}[M_{tt}]^{N_{tt}}$,
\end{center}
where $C(x^n,k)$ is a constant that only depends on $x^n$ and $k$. 

Hence, for any $x^n\in{K_\ell(i)}$,
\begin{align*}
\hat{P}^*_{x^n}(j)
&=\frac{\displaystyle\sum_{(M)\in \mathscr{P}_S}{\displaystyle{M_{ij}}\prod_{\substack{t=2\\ t\ even}}^{k}} \left[M_{t(t-1)}\right]^{N_{t(t-1)}}[M_{tt}]^{N_{tt}}}{\displaystyle\sum_{(M)\in \mathscr{P}_S}{\displaystyle\prod_{\substack{t=2\\ t\ even}}^{k}} \left[M_{t(t-1)}\right]^{N_{t(t-1)}}[M_{tt}]^{N_{tt}}}.
\end{align*}

Below we show how to evaluate $\hat{P}^*_{x^n}(j)$ for $j=i\in[k]^e$, and other cases can be derived similarly.

Combining ${M_{jj}}^{N_{jj}}$ with $M_{jj}$ in the nominator,
\begin{align*}
\hat{P}^*_{x^n}(j)
&=
\frac{\displaystyle\sum_{(M)\in \mathscr{P}_S}{\displaystyle{\left[M_{jj}^{\ell}\right]}\prod_{\substack{t=2\\ t\ even\\t\not=j}}^{k}} \left[M_{t(t-1)}\right]^{N_{t(t-1)}}[M_{tt}]^{N_{tt}}}{\displaystyle\sum_{(M)\in \mathscr{P}_S}{\displaystyle{\left[M_{jj}^{\ell-1}\right]}\prod_{\substack{t=2\\ t\ even\\t\not=j}}^{k}} \left[M_{t(t-1)}\right]^{N_{t(t-1)}}[M_{tt}]^{N_{tt}}}
\\&=
\frac{\displaystyle\sum_{\substack{v\in{V_{n}}\\v'\in{V_{n}}}}{\displaystyle{(b-v')^{\ell}}\prod_{\substack{t=2\\ t\ even\\t\not=j}}^{k}} v^{N_{t(t-1)}}(b-v)^{N_{tt}}}{\displaystyle\sum_{\substack{v\in{V_{n}}\\v'\in{V_{n}}}}{\displaystyle{(b-v')^{\ell-1}}\prod_{\substack{t=2\\ t\ even\\t\not=j}}^{k}} v^{N_{t(t-1)}}(b-v)^{N_{tt}}}
\\&=
\frac{{\left[\sum_{\substack{v'\in{V_{n}}}}(b-v')^{\ell}\right]}\displaystyle\sum_{\substack{v\in{V_{n}}}} {\displaystyle\prod_{\substack{t=2\\ t\ even\\t\not=j}}^{k}}v^{N_{t(t-1)}}(b-v)^{N_{tt}}}{{\left[\sum_{\substack{v'\in{V_{n}}}}(b-v')^{\ell-1}\right]}\displaystyle\sum_{\substack{v\in{V_{n}}}} {\displaystyle\prod_{\substack{t=2\\ t\ even\\t\not=j}}^{k}}v^{N_{t(t-1)}}(b-v)^{N_{tt}}}
\\&=
\frac{\sum_{\substack{v\in{V_{n}}}}(b-v)^{\ell}}{\sum_{\substack{v\in{V_{n}}}}(b-v)^{\ell-1}}.
\end{align*}
This completes the proof.
\end{proof}

\subsection{Analysis}\label{analysis}
Next, for any $x^n\in{K_\ell(i)}$, we lower bound $D_{\text{KL}}(P_{x^n}, \hat{P}^*_{x^n})$ in terms of $M_{i(i-1)}$ and $ \hat{P}^*_{x^n}(i-1)$.

\begin{Lemma}\label{lemma3}
For any $(M)\in \mathscr{P}_S$ and $x^n\in{K_\ell(i)}$,
\[
D_{\text{KL}}(P_{x^n}, \hat{P}^*_{x^n})\geq M_{i(i-1)}\Paren{-1+\log\frac{M_{i(i-1)}}{\hat{P}^*_{x^n}(i-1)}}.
\]
\end{Lemma}
\begin{proof}
By the previous lemma, 
\[
D_{\text{KL}}(P_{x^n}, \hat{P}^*_{x^n})=M_{ii}\log\frac{M_{ii}}{\hat{P}^*_{x^n}(i)}+M_{i(i-1)}\log\frac{M_{i(i-1)}}{\hat{P}^*_{x^n}(i-1)}.
\]
Noting that $\frac{x}{x+1}\leq \log(x+1)$ for all $x>-1$,
\begin{align*}
M_{ii}\log\frac{M_{ii}}{\hat{P}^*_{x^n}(i)}
&=M_{ii}\log\Paren{\frac{M_{ii}-\hat{P}^*_{x^n}(i)}{\hat{P}^*_{x^n}(i)}+1}\\
&\geq M_{ii}-\hat{P}^*_{x^n}(i)\\
&= \Paren{b-M_{i(i-1)}}- \Paren{b-\hat{P}^*_{x^n}{(i-1)}}\\
&\geq -M_{i(i-1)}.
\end{align*}
This completes the proof.
\end{proof}

Let $V'_n:=\{\frac{1}{(\log n)^t}\ \mid t\in\mathds{N}, 1\leq{t}\leq{\frac{\log n}{4\log \log n}\}}$ be a subset of $V_n$ whose size is $\frac{1}{2}|V_n|$. For $M_{i(i-1)}\in V'_n$, we further lower bound ${M_{i(i-1)}}/{\hat{P}^*_{x^n}(i-1)}$ in terms of $n$. 

Let $\ell_1(M):=\frac{1}{M_{i(i-1)}}\frac{1}{\log\log n}$ and $\ell_2(M):= {\frac{1}{M_{i(i-1)}}\log\log n}$, we have 
\begin{Lemma}\label{lemma4}
For any $(M)\in \mathscr{P}_S$, $x^n\in{K_\ell(i)}$ where $i\in[k]^e$, $M_{i(i-1)}=\frac{1}{(\log n)^m}\in V'_n$, and sufficiently large $n$, if 
\[
\ell_1(M)\leq \ell \leq \ell_2(M),
\]
then,
\[
\frac{M_{i(i-1)}}{\hat{P}^*_{x^n}(i-1)}\gtrsim \frac{\log n}{8\log\log n}(1-o(1)).
\]
\end{Lemma}

\begin{proof}
Consider $M_{i(i-1)}=\frac{1}{(\log n)^m}\in V'_n$, where $m\in [1, \frac{\log n}{4\log \log n}]$.

Note that for $x^n\in{K_\ell(i)}$, the value of $\hat{P}^*_{x^n}(i-1)$ only depends on $\ell$, we can define
\[
F_\ell:=\frac{M_{i(i-1)}}{\hat{P}^*_{x^n}(i-1)}.
\]
We have
\[
F_\ell\geq\frac{A_\ell+X_\ell+C_\ell}{B_\ell+X_\ell+D_\ell},
\]
where
\begin{align*}
&X_\ell:=\Paren{1-\frac{k-2}{n}-\frac{1}{(\log{n})^m}}^\ell,\\&
A_\ell:=\sum\limits_{i=1}^{m-1}{\Paren{1-\frac{k-2}{n}-\frac{1}{(\log{n})^i}}^\ell},\\&
C_\ell:=\sum\limits_{i=m+1}^{\frac{\log n}{2\log \log n}}{\Paren{1-\frac{k-2}{n}-\frac{1}{(\log{n})^i}}^\ell},\\&
B_\ell:=\sum\limits_{i=1}^{m-1}{\Paren{1-\frac{k-2}{n}-\frac{1}{(\log{n})^i}}^\ell}{(\log n)^{m-i}},\\&
\text{ and }D_\ell:=\sum\limits_{i=m+1}^{\frac{\log n}{2\log \log n}}{\Paren{1-\frac{k-2}{n}-\frac{1}{(\log{n})^i}}^\ell}{(\log n)^{m-i}}.
\end{align*}

We have the following bounds on these quantities.

\subsection*{Bounds for $X_\ell$}
\[
0\leq X_\ell = \Paren{1-\frac{k-2}{n}-\frac{1}{(\log{n})^m}}^\ell\leq 1.
\]

\subsection*{Bounds for $A_\ell$}
\[
0\leq A_\ell = \sum\limits_{i=1}^{m-1}{\Paren{1-\frac{k-2}{n}-\frac{1}{(\log{n})^i}}^\ell}.
\]

\subsection*{Bounds for $D_\ell$}
\[
0\leq D_\ell \leq{\sum\limits_{i=m+1}^{\frac{\log n}{2\log \log n}}{\Paren{1-\frac{k-2}{n}-\frac{1}{(\log{n})^i}}^\ell}{\frac{1}{\log n}}}={\frac{1}{\log n}C_\ell}.
\]

\subsection*{Bounds for $C_\ell$}

Note that
\[
\frac{(\log n)^m}{\log\log n}\leq \ell \leq {(\log n)^m \log\log n}
\]
and 
\[
(\log n)^m\leq \sqrt{n}.
\]

Consider a single term of $C_\ell$, we have
\begin{align*}
\Paren{1-\frac{k-2}{n}-\frac{1}{(\log n)^i}}^\ell
&\geq{\Paren{1-\frac{k-2}{n}-\frac{1}{(\log n)^i}}^{(\log n)^m\log \log n}}\\
&={\Paren{1-\frac{k-2}{n}-\frac{1}{(\log n)^i}}^{\frac{1}{\frac{k-2}{n}+\frac{1}{(\log n)^i}}\Paren{\frac{k-2}{n}+\frac{1}{(\log n)^i}}\Paren{\log n}^m\log \log n}}\\
&\geq{\left[\left(1-\frac{k-2}{n}-\frac{1}{(\log n)^i}\right)^{\frac{1}{\frac{k-2}{n}+\frac{1}{(\log n)^i}}}\right]^{(\frac{k-2}{\sqrt{n}}+\frac{1}{\log n})\log\log n}}\\
&\geq{\Paren{\frac{1}{4}}^{\frac{(k-2)\log\log n}{\sqrt{n}}+\frac{\log \log n}{\log n}}}\\
&\geq{\Paren{\frac{1}{4}}^\frac{1}{2}}=\frac{1}{2},
\end{align*}
where we use the inequality $i\geq m+1$ and $(1-\frac{1}{x})^x\geq{\frac{1}{4}}$ for $x\geq{2}$. 

Hence,
\[
{\frac{\log n}{8\log \log n}}={\frac{\log n}{4\log \log n}\cdot\frac{1}{2}}\leq C_\ell = \sum\limits_{i=m+1}^{\frac{\log n}{2\log \log n}}{\Paren{1\!\!-\!\!\frac{k-2}{n}\!\!-\!\!\frac{1}{(\log{n})^i}}^\ell} \leq{\sum\limits_{i=m+1}^{\frac{\log n}{2\log \log n}}{1}}\leq{\frac{\log n}{2\log \log n}}.
\]

\subsection*{Bounds for $B_\ell$}
Similarly, consider a single term of $B_\ell$ without the factor $(\log n)^{m-i}$, 
\begin{align*}
\Paren{1-\frac{k-2}{n}-\frac{1}{(\log n)^i}}^\ell
&\leq{\Paren{1-\frac{k-2}{n}-\frac{1}{(\log{n})^i}}}^{\frac{(\log{n})^m}{\log{\log{n}}}}\\
&\leq{\Paren{1-\frac{k-2}{n}-\frac{1}{(\log{n})^i}}}^{\frac{1}{\frac{1}{(\log)^i}+\frac{k-2}{n}}\Paren{\frac{1}{(\log)^i}+\frac{k-2}{n}}\frac{(\log{n})^m}{\log{\log{n}}}}\\
&\leq{\left[\left(1-\frac{k-2}{n}-\frac{1}{(\log{n})^i}\right)^{\frac{1}{\frac{1}{(\log)^i}+\frac{k-2}{n}}}\right]^{\Paren{\frac{1}{(\log)^i}+\frac{k-2}{n}}\frac{(\log{n})^m}{\log{\log{n}}}}}\\
&\leq{\Paren{\frac{1}{e}}^\frac{(\log n)^{m-i}}{\log{\log{n}}}}\\
&=\Paren{\frac{1}{n}}^\frac{(\log n)^{m-i-1}}{\log{\log{n}}},
\end{align*}
where we use the inequality $(1-\frac{1}{x})^x\leq{\frac{1}{e}}$ for $x\geq{2}$. 
\vspace{+1em}

\vfill
\pagebreak
Hence,
\begin{align*}
B_\ell
&=\sum\limits_{i=1}^{m-1}{\Paren{1-\frac{k-2}{n}-\frac{1}{(\log{n})^i}}^\ell}{(\log n)^{m-i}}\\
&\leq{\sum\limits_{i=1}^{m-1}\Paren{\frac{1}{n}}^\frac{(\log n)^{m-i-1}}{\log{\log{n}}}{(\log n)^{m-i}}}\\
&=\sum\limits_{i=1}^{m-1}\Paren{\frac{1}{n}}^\frac{(\log n)^{m-i-1}}{3\log{\log{n}}}{(\log n)^{m-i}}\Paren{\frac{1}{n}}^\frac{2(\log n)^{m-i-1}}{3\log{\log{n}}}\\
&={\Paren{\frac{1}{n}}^\frac{1}{\log{\log{n}}}\log n+\sum\limits_{i=1}^{m-2}\Paren{\frac{1}{n}}^\frac{(\log n)^{m-i-1}}{3\log{\log{n}}}{(\log n)^{m-i}}\Paren{\frac{1}{n}}^\frac{2(\log n)^{m-i-1}}{3\log{\log{n}}}}\\
&\leq{\Paren{\frac{1}{n}}^\frac{1}{\log{\log{n}}}\log n+\sum\limits_{i=1}^{m-2}\Paren{\frac{1}{n}}^\frac{\log n}{3\log{\log{n}}}{(\log n)^{m}}\Paren{\frac{1}{n}}^\frac{2\log n}{3\log{\log{n}}}}\\
&\leq{\Paren{\frac{1}{n}}^\frac{1}{\log{\log{n}}}\log n+\sum\limits_{i=1}^{m-2}\Paren{\frac{1}{n}}^\frac{\log n}{3\log{\log{n}}}{(\log n)^{\frac{\log n}{4\log{\log{n}}}}}\Paren{\frac{1}{n}}^\frac{2\log n}{3\log{\log{n}}}}\\
&\leq{\Paren{\frac{1}{n}}^\frac{1}{\log{\log{n}}}\log n+\frac{\log n}{4\log{\log{n}}}\Paren{\frac{1}{n}}^\frac{\log n}{3\log{\log{n}}}{(\log n)^{\frac{\log n}{4\log{\log{n}}}}}\Paren{\frac{1}{n}}^\frac{2\log n}{3\log{\log{n}}}}
\\
&\leq{\Paren{\frac{1}{n}}^\frac{1}{\log{\log{n}}}\log n+\Paren{\frac{1}{n}}^\frac{\log n}{3\log{\log{n}}}{(\log n)^{\frac{\log n}{4\log{\log{n}}}+1}}\Paren{\frac{1}{n}}^\frac{2\log n}{3\log{\log{n}}}}\\
&\leq{\Paren{\frac{1}{n}}^\frac{1}{\log{\log{n}}}e^{\log\log n}+\Paren{\frac{\log{n}}{n}}^\frac{\log n}{3\log{\log{n}}}\Paren{\frac{1}{n}}^\frac{2\log n}{3\log{\log{n}}}}\\
&\leq{e^{\frac{-\log{n}+(\log{\log{n}})^2}{\log\log n}}+\frac{1}{n}}\\
&\leq{e^{-\frac{2(\log{\log{n}})^2+(\log{\log{n}})^2}{\log\log n}}+\frac{1}{n}}\\
&\leq{e^{-\log{\log{n}}}+\frac{1}{n}}\\
&\leq{\frac{2}{\log{n}}},
\end{align*}
where we use the inequality $x-2(\log{x})^2\geq{0}$ for $x\geq{1}$.

Putting everything together:
\begin{align*}
F_\ell&=\frac{A_\ell+X_\ell+C_\ell}{B_\ell+X_\ell+D_\ell}\geq{\frac{0+\frac{\log n}{8\log \log n}}{\frac{2}{\log{n}}+1+\frac{1}{{2\log \log n}}}}\asymp\frac{\log n}{8\log \log n}.
\end{align*}
This completes the proof.
\end{proof}

Another quantity that will be appear later is $\Pr(X^n\in K_{\ell}(i))$ where $X^n\sim (M)\in \mathscr{P}_S$.
We need the following lower bound.
\begin{Lemma}\label{lemma5}
For $X^n\sim (M)\in \mathscr{P}_S$ and $i\in[k]^e$,
\[
\Pr(X^n\in K_{\ell}(i))\gtrsim \frac{k-1}{ek}\frac{1}{n}\Paren{1-\frac{k-2}{n}-M_{i(i-1)}}^{l-1}.
\]
\end{Lemma}
\begin{proof}
By our construction of $ \mathscr{P}_S$, for $X^n\sim (M)\in \mathscr{P}_S$ and $i\in[k]^e$, we have the following observations.
\begin{enumerate}
\item The probability that the initial state is not $i$ is $\frac{k-1}{k}$.
\item The probability of transitioning from some state $j\not=i$ to some state that is not $i$ is $1-\frac{1}{n}$.
\item The probability of transitioning from some state $j\not=i$ to state $i$ is $\frac{1}{n}$.
\item The probability of transitioning from state $i$ to itself is $1-\frac{k-2}{n}-M_{i(i-1)}$.
\end{enumerate}
Therefore,
\begin{align*}
\Pr(X^n\in K_{\ell}(i))
&=\frac{k-1}{k} \Paren{1-\frac{1}{n}}^{n-\ell-1}\frac{1}{n}\Paren{1-\frac{k-2}{n}-M_{i(i-1)}}^{\ell-1}\\
&\geq \frac{k-1}{k} \Paren{1-\frac{1}{n}}^{n}\frac{1}{n}\Paren{1-\frac{k-2}{n}-M_{i(i-1)}}^{\ell-1}\\
&\asymp \frac{k-1}{ek}\frac{1}{n}\Paren{1-\frac{k-2}{n}-M_{i(i-1)}}^{\ell-1}.
\end{align*}
This completes the proof.
\end{proof}

Now we turn back to $\rho_n^{\text{KL}}(\mathscr{P})$. According to the previous derivations,
\begin{align*}
\rho_n^{\text{KL}}(\mathscr{P})
&\geq{ \min_{\hat{P}}\EE_{P\sim U(\mathscr{P}_S)} [\rho_n^\text{KL}(P,\hat{P}; \mathscr{K}_n)]}\\
&=\EE_{P\sim U(\mathscr{P}_S)} \left[ \sum_{x^n\in \mathscr{K}_n}\Pr_{X^n\sim P}(X^n=x^n) D_{\text{KL}}(P_{x^n},\hat{P}^*_{x^n}) \right]\\
&=\frac{1}{|\mathscr{P}_S|}\sum\limits_{(M)\in \mathscr{P}_S}\sum\limits_{l=1}^{n-1}\sum\limits_{i\in [k]}\sum\limits_{x^n\in{K_\ell(i)}}\left[\Pr_{X^n\sim P}(X^n=x^n) D_{\text{KL}}(P_{x^n},\hat{P}^*_{x^n})\right]\\
&\geq \frac{1}{|\mathscr{P}_S|}\sum\limits_{(M)\in \mathscr{P}_S}\sum\limits_{\ell=\ell_1(M)}^{\ell_2(M)}\sum\limits_{i\in [k]^e}\sum\limits_{x^n\in{K_\ell(i)}}\left[\Pr_{X^n\sim P}(X^n=x^n) D_{\text{KL}}(P_{x^n},\hat{P}^*_{x^n})\right].
\end{align*}
Noting that all $x^n\in{K_\ell(i)}$ have the same $P_{x^n}$ and $\hat{P}^*_{x^n}$, thus, the last formula can be written as
\[
\frac{1}{|\mathscr{P}_S|}\sum\limits_{(M)\in \mathscr{P}_S}\sum\limits_{\ell=\ell_1(M)}^{\ell_2(M)}\sum\limits_{i\in [k]^e}\left[\Pr_{X^n\sim P}(X^n\in{K_\ell(i)}) D_{\text{KL}}(P_{x^n},\hat{P}^*_{x^n}; x^n \in{K_\ell(i)})\right].
\]
By Lemma~\ref{lemma3} and~\ref{lemma4}, for $\ell_1(M)\leq \ell\leq \ell_2(M)$ and $M_{i(i-1)}\in V'_n$, 
\begin{align*}
D_{\text{KL}}(P_{x^n},\hat{P}^*_{x^n}; x^n \in{K_\ell(i)})
&\geq M_{i(i-1)}\Paren{-1+\log\frac{M_{i(i-1)}}{\hat{P}^*_{x^n}(i-1)}}\\
&\gtrsim M_{i(i-1)}\Paren{-1+\log\Paren{\frac{\log n}{8\log\log n}}}\\
&\asymp M_{i(i-1)}\log\log n.
\end{align*}
By Lemma~\ref{lemma5},
\[
\Pr(X^n\in K_{\ell}(i))\gtrsim \frac{k-1}{ek}\frac{1}{n}\Paren{1-\frac{k-2}{n}-M_{i(i-1)}}^{\ell-1}.
\]
Therefore,
\begin{align*}
\rho_n^{\text{KL}}(\mathscr{P})
&\geq \frac{1}{|\mathscr{P}_S|}\sum\limits_{(M)\in \mathscr{P}_S}\sum\limits_{\ell=\ell_1(M)}^{\ell_2(M)}\sum\limits_{i\in [k]^e}\left[\Pr_{X^n\sim P}(X^n\in{K_\ell(i)}) D_{\text{KL}}(P_{x^n},\hat{P}^*_{x^n}; x^n \in{K_\ell(i)})\right]\\
&\gtrsim \frac{(k-1)\log\log n}{enk}\sum\limits_{i\in [k]^e} \frac{1}{|\mathscr{P}_S|}\!\!\!\!\sum\limits_{\substack{(M)\in \mathscr{P}_S\\ \text{ and } M_{i(i-1)}\in V_n'}}\sum\limits_{\ell=\ell_1(M)}^{\ell_2(M)}\Paren{1\!\!-\!\!\frac{k-2}{n}\!\!-\!\!M_{i(i-1)}}^{\ell-1}\!\!\!\!M_{i(i-1)}\\
&\geq \frac{(k-1)\log\log n}{enk}\sum\limits_{i\in [k]^e}\frac{1}{|V_n|}\sum_{v\in V'_n}\sum\limits_{\ell=\frac{1}{v}\frac{1}{\log\log n}}^{\frac{1}{v}\log\log n}\Paren{1\!\!-\!\!\frac{k-2}{n}\!\!-\!\!v}^{\ell-1}v,
\end{align*}
where the last step follows by symmetry.

Next, we show that for any $v=\frac{1}{(\log n)^m}\in V'_n$,
\[
T_m:= \sum\limits_{\ell=\frac{1}{v}\frac{1}{\log\log n}}^{\frac{1}{v}\log\log n}\Paren{1-\frac{k-2}{n}-v}^{\ell-1}v\gtrsim 1.
\]

Noting that $T_m$ is simply the summation of a geometric sequence, we can compute it as follows
\begin{equation*}\label{eq:U_regret47}
\begin{split}
T_m
&=\frac{1}{(\log n)^m}\sum\limits_{\ell=\frac{(\log n)^m}{\log \log n}}^{(\log n)^m\log \log n}\left[\left(1-\frac{k-2}{n}-\frac{1}{(\log n)^m}\right)^{\ell-1}\right]\\
&=\frac{1}{(\log n)^m}\frac{\left(1\!\!-\!\!\frac{k-2}{n}\!\!-\!\!\frac{1}{(\log n)^m}\right)^{\frac{(\log n)^m}{\log \log n}-1}-\left(1\!\!-\!\!\frac{k-2}{n}\!\!-\!\!\frac{1}{(\log n)^m}\right)^{(\log n)^m\log \log n}}{1-\left(1\!\!-\!\!\frac{k-2}{n}\!\!-\!\!\frac{1}{(\log n)^m}\right)}\\
&=\frac{1}{\frac{(k-2)(\log{n})^m}{n}+1}\left[\left(1-\frac{k-2}{n}-\frac{1}{(\log n)^m}\right)^{\frac{(\log n)^m}{\log \log n}-1}\right.\\& \left.-\left(1-\frac{k-2}{n}-\frac{1}{(\log n)^m}\right)^{(\log n)^m\log \log n}\right].
\end{split}
\end{equation*}

To provide a lower bound for $T_m$, we use the following inequalities:
 
\begin{align*}
\frac{1}{\frac{(k-2)(\log{n})^m}{n}+1}\geq{\frac{1}{\frac{(k-2)(\log{n})^{\frac{\log n}{4\log \log n}}}{n}+1}}=\frac{1}{\frac{(k-2)n^{\frac{1}{4}}}{n}+1}\asymp 1,
\end{align*}

\begin{align*}
\left(1\!\!-\!\!\frac{k-2}{n}\!\!-\!\!\frac{1}{(\log n)^m}\right)^{\frac{(\log n)^m}{\log \log n}-1}
&\geq\left[{\left(1\!\!-\!\!\frac{k-2}{n}\!\!-\!\!\frac{1}{(\log{n})^m}\right)}^{\frac{1}{\frac{1}{(\log)^m}+\frac{k-2}{n}}}\right]^{\left(1+\frac{(k-2)(\log{n})^m}{n}\right)\frac{1}{\log{\log{n}}}}\\
&\geq{\Paren{\frac{1}{4}}^{(1+\frac{(k-2)\sqrt{n}}{n})\frac{1}{\log{\log{n}}}}}\geq{\Paren{\frac{1}{4}}^{2\frac{1}{\log{\log{n}}}}}\asymp 1,
\end{align*}
and
\begin{align*}
&\left(1\!\!-\!\!\frac{k-2}{n}\!\!-\!\!\frac{1}{(\log n)^m}\right)^{(\log n)^m\log \log n}\\
&=\left[\left(1\!\!-\!\!\frac{k-2}{n}\!\!-\!\!\frac{1}{(\log n)^m}\right)^{\frac{1}{\frac{k-2}{n}+\frac{1}{(\log n)^m}}}\right]^{\left(\frac{(k-2)(\log{n})^m}{n}+1\right)\log \log n}\\
&\leq{\left(\frac{1}{e}\right)^{\log \log n}}=\frac{1}{\log{n}}.
\end{align*}

\noindent Consolidating these three inequalities, the sum $T_m$ can be lower bounded by
\[
T_m\gtrsim{{1}}{(1-\frac{1}{\log{n}})}\asymp1.
\]
Finally,
\begin{align*}
\rho_n^{\text{KL}}(\mathscr{P})
&\gtrsim \frac{(k-1)\log\log n}{enk}\sum\limits_{i\in [k]^e}\frac{1}{|V_n|}\sum_{v\in V'_n}(1-o(1))\\
&= \frac{(k-1)\log\log n}{enk}\frac{k}{2}\frac{|V'_n|}{|V_n|}\\
&= \frac{(k-1)\log\log n}{4en}.
\end{align*}
\section{Minimax prediction: upper bound}
The proof makes use of the following lemma, which provides a uniform upper bound for the hitting probability of any $k$-state Markov chain.
\begin{Lemma}~\cite{sup17}\label{lemma6}
For any Markov chain over $[k]$ and any two states $i,j\in[k]$, if $n>k$, then
\[
\textrm{Pr}_i{(\tau(j)=n)}\leq \frac{k}{n}.
\]
\end{Lemma}
Let $\mathscr{K}_n$ be the same as is in the previous section. Recall that
\[
\rho_n^{\text{KL}}(P,\hat{P}; \mathscr{K}_n) = \sum_{x^n\in \mathscr{K}_n}P(x^n)D_{\text{KL}}(P_{x^n},\hat{P}_{x^n}),
\]
we denote the \emph{partial minimax prediction risk over $\mathscr{K}_n$} by
\[
\rho_n^{\text{KL}}(\mathscr{P};  \mathscr{K}_n):=\min_{\hat{P}}\max_{P\in \mathscr{P}}\rho_n^{\text{KL}}(P,\hat{P}; \mathscr{K}_n).
\]
Let $\overline{\mathscr{K}_n}:=[k]^n\setminus\mathscr{K}_n$, we define $\rho_n^{\text{KL}}(P,\hat{P}; \overline{\mathscr{K}_n})$ and $\rho_n^{\text{KL}}(\mathscr{P};  \overline{\mathscr{K}_n})$ in the same manner. As the consequence of $\hat{P}$ being a function from $[k]^n$ to $\Delta_k$, we have the following triangle inequality,
\[
\rho_n^{\text{KL}}(\mathscr{P}) \leq  \rho_n^{\text{KL}}(\mathscr{P}; \overline{\mathscr{K}_n}) + \rho_n^{\text{KL}}(\mathscr{P}; \mathscr{K}_n).
\]
Turning back to Markov chains, the next lemma upper bounds $\rho_n^{\text{KL}}({\mathds{M}}^k; \overline{\mathscr{K}_n})$.
\begin{Lemma}\label{lemma7}
Let $\hat{P}^{+\frac{1}{2}}$ denote the estimator that maps $X^n\sim (M)$ to $\hat{M}^{+\frac{1}{2}}(X_{n}, \cdot)$, then
\[
\max_{P\in{\mathds{M}}^k}\rho_n^{\text{KL}}(P, \hat{P}^{+\frac{1}{2}};  \overline{\mathscr{K}_n})\leq{\mathcal{O}_k\left(\frac{1}{n}\right)},
\]
which implies
\[
\rho_n^{\text{KL}}({\mathds{M}}^k; \overline{\mathscr{K}_n})\leq{\mathcal{O}_k\left(\frac{1}{n}\right)}.
\]
\end{Lemma}
\begin{proof}
The proof of this lemma is essentially a combination of the upper bounds' proofs in~\cite{Moein16} and in Section~\ref{sec4}. Instead of using the fact that $M_{ij}$ are bounded away from $0$ (see Section~\ref{sec4}), we partition $\overline{\mathscr{K}_n}$ into different subsets according to how close the counts are to their expected values, the number of times that the last appearing state transitioning to itself, and the number of times that the last appearing state transitioning to other states.  Then, we bound the estimator's expected loss over each set of the partition by ${\mathcal{O}_k\left({1}/{n}\right)}$. We omit the proof for the sake of brevity.
\end{proof}

Recall the following lower bound, 
\[
\rho_n^{\text{KL}}({\mathds{M}}^k)=\Omega_k\Paren{\frac{\log\log n}{n}}.
\]
This together with Lemma~\ref{lemma3} and the triangle inequality above shows that an upper bound on $\rho_n^{\text{KL}}({\mathds{M}}^k; \mathscr{K}_n)$ also suffices to bound the leading term of $\rho_n^{\text{KL}}({\mathds{M}}^k)$. The following lemma provides such an upper bound. Recall that for any $i\in[k]$, $K_\ell(i)$ is defined as $\{x^n\in[k]^n: x^n={\bar{i}}^{n-\ell} i^\ell\}$.
\begin{Lemma}\label{lemma8}
For any $x^n\in \mathscr{K}_n$, there exists a unique pair $(\ell, i)$ such that $x^n\in K_\ell(i)$. Consider the following estimator 
\[ 
\hat{P}_{x^n}(i)
:=
\begin{cases}  
      1-\frac{1}{\ell\log n}& \ell\leq\frac{n}{2}\\ 
      1-\frac{1}{\ell}&  \ell>\frac{n}{2}
   \end{cases}
\]
and 
\[ 
\hat{P}_{x^n}(j)
:=\frac{1-\hat{P}_{x^n}(i)}{k-1},\ \forall j\in[k]\setminus\{i\},
\]
then we have 
\[
\rho_n^{\text{KL}}({\mathds{M}}^k;  \mathscr{K}_n)\leq \max_{P\in{\mathds{M}}^k}\rho_n^{\text{KL}}(P,\hat{P}; \mathscr{K}_n)\lesssim\frac{2k^2 \log\log n}{n}.
\]
\end{Lemma}

\begin{proof}
Let $i\in [k]$ be an arbitrary state. For simplicity of illustration, we use the following notation: for any $x^n={\bar{i}}^{n-\ell} i^\ell$, denote $\hat{p}_\ell:=\hat{P}_{x^n}$; for any $(M)\in{\mathds{M}}^k$, denote $p_i:=M(i,\cdot)$; for any $\ell\leq{n}$, denote $h_{i,\ell}:=\Pr(\tau(i)=\ell)$. By Lemma~\ref{lemma6}, the hitting probability $h_{i,\ell}$ is upper bounded by $k/\ell$ for all $\ell>k$. We can write
\[
\rho_n^{\text{KL}}(P,\hat{P}; \mathscr{K}_n)=\sum_{i\in[k]}\sum_{\ell=1}^{n-1} h_{i,n-\ell}(p_i(i))^{\ell-1} D_{\text{KL}}(p_i, \hat{p}_\ell).
\]
Now, we break the right hand side into two sums according to whether $\ell$ is greater than $n/2$ or not.
For $\ell> n/2$, we have
\begin{align*}
&\sum_{i\in[k]}\sum_{\ell=\frac{n}{2}+1}^{n-1} h_{i,n-\ell}(p_i(i))^{\ell-1} D_{\text{KL}}(p_i, \hat{p}_\ell)\\
&\leq \sum_{i\in[k]}\sum_{\ell=\frac{n}{2}+1}^{n-1} h_{i,n-\ell}(p_i(i))^{\ell-1} \Paren{p_i(i)\log \Paren{\frac{p_i(i)}{1-\frac{1}{\ell}}}+\sum_{j\not=i}p_i(j)\log\Paren{\frac{\sum_{j\not=i}p_i(j)}{\frac{1}{\ell(k-1)}}}}\\
&\leq \sum_{i\in[k]}\sum_{\ell=\frac{n}{2}+1}^{n-1} h_{i,n-\ell}(p_i(i))^{\ell-1} \Paren{\log \Paren{\frac{1}{1-\frac{1}{\ell}}}+(1-p_i(i))\log\Paren{{\ell(k-1)(1-p_i(i))}}}\\
&\leq \sum_{i\in[k]}\sum_{\ell=\frac{n}{2}+1}^{n-1} h_{i,n-\ell}(p_i(i))^{\ell-1} \Paren{\frac{\frac{1}{\ell}}{1-\frac{1}{\ell}}+(1-p_i(i))^2{\ell(k-1)}}\\
&\leq \sum_{i\in[k]}\sum_{\ell=\frac{n}{2}+1}^{n-1} h_{i,n-\ell}\Paren{\frac{2}{n}+(p_i(i))^{\ell-1}(1-p_i(i))^2{\ell(k-1)}}\\
&\leq \sum_{i\in[k]}\sum_{\ell=\frac{n}{2}+1}^{n-1} h_{i,n-\ell}\Paren{\frac{2}{n}+\frac{1}{(\ell+1)^2}{\ell(k-1)}}\\
&\leq \sum_{i\in[k]}\sum_{\ell=\frac{n}{2}+1}^{n-1} h_{i,n-\ell}\Paren{\frac{2k}{n}}\\
&= \sum_{i\in[k]}\frac{2k}{n}\Pr(\tau(i)\in[1, n/2-1])
\leq \frac{2k^2}{n}.
\end{align*}
Similarly, for $\ell\leq n/2$, we have
\begin{align*}
&\sum_{i\in[k]}\sum_{\ell=1}^{\frac{n}{2}} h_{i,n-\ell}(p_i(i))^{\ell-1} D_{\text{KL}}(p_i, \hat{p}_\ell)\\
&\leq \sum_{i\in[k]}\sum_{\ell=1}^{\frac{n}{2}} h_{i,n-\ell} (p_i(i))^{\ell-1} \Paren{\log \Paren{\frac{1}{1-\frac{1}{\ell\log n}}}+(1-p_i(i))\log\Paren{{\ell(k-1)(1-p_i(i))\log n}}}\\
&\leq \sum_{i\in[k]}\sum_{\ell=1}^{\frac{n}{2}} \frac{2k}{n}(p_i(i))^{\ell-1} \Paren{\frac{2}{\ell\log n}+(1-p_i(i))^2{\ell(k-1)}+(1-p_i(i))\log\log n}\\
&\leq \sum_{i\in[k]} \frac{2k}{n}\Paren{\sum_{\ell=1}^{\frac{n}{2}} \frac{2}{\ell\log n}+\sum_{\ell=1}^{\frac{n}{2}} \ell(p_i(i))^{\ell-1}(1-p_i(i))^2{(k-1)}+\sum_{\ell=1}^{\frac{n}{2}}(p_i(i))^{\ell-1}(1-p_i(i))\log\log n}\\
&\leq \sum_{i\in[k]} \frac{2k}{n}\Paren{2+(k-1)+\log\log n}\\
&\asymp \frac{2k^2 \log\log n}{n}.
\end{align*}
This completes the proof.
\end{proof}

\section{Minimax estimation: lower bound}
The proof of the lower bound makes use of the following concentration inequality, which upper bounds the probability that a binomial random variable exceeds its mean.
\begin{Lemma}~\cite{com06}\label{lemma9}
Let $Y$ be a binomial random variable with parameters $m\in\mathds{N}$ and $p\in[0,1]$, then for any $\epsilon\in(0,1)$,
\[
\textrm{Pr}{(Y\geq(1+\epsilon) mp)}\leq \exp\left(-\epsilon^2mp/3\right).
\]
\end{Lemma}

\subsection{Prior construction}
Again we use the following standard argument to lower bound the minimax risk,
\[
\varepsilon_n^L(\mathscr{M})=\min_{\hat{M}}\max_{(M)\in \mathscr{M}} \varepsilon_n^L(M,\hat{M})\geq{ \min_{\hat{M}}\EE_{(M)\sim U(\mathscr{M}_S)} [\varepsilon_n^L(M,\hat{M})]},
\]
where $\mathscr{M}_S\subset \mathscr{M}$ and $U(\mathscr{M}_S)$ is the uniform distribution over $\mathscr{M}_S$. Setting $\mathscr{M}={\mathds{M}}^k_{\delta, \pi^*}$, we outline the construction of $\mathscr{M}_S$ as follows.

We adopt the notation in~\cite{KamathOPS15} and denote the $L_{\infty}$ ball of radius $r$ around $u_{k-1}$, the uniform distribution over $[k-1]$, by
\[
B_{k-1}(r):=\{p\in \Delta_{k-1}: {L_\infty}(p, u_{k-1})<r \},
\]
where ${L_\infty}(\cdot,\cdot)$ is the $L_\infty$ distance between two distributions.
For simplicity, define
\[
p':=(p_1,\ p_2,\ \ldots,\ p_{k-1}),
\]
\[
p^*:= \Paren{\nor,\ \nor,\  \dots\ \nor,\ \pi^*},
\]
and
\[
M_n(p') 
:=
\begin{bmatrix}
    \nor & \nor  &  \dots  & \nor  & \pi^*\\
    \nor & \nor  &  \dots  & \nor & \pi^*\\
    \vdots & \vdots  & \ddots & \vdots &\vdots\\
    \nor & \nor   & \dots  & \nor & \pi^*\\
   \bar{\pi}^*p_1 &  \bar{\pi}^*p_2 & \dots  & \bar{\pi}^*p_{k-1} & \pi^*
\end{bmatrix},
\]
where $\bar{\pi}^*=1-\pi^*$ and $\sum_{i=1}^{k-1} p_i=1$.

Given $n$ and $\epsilon\in(0,1)$, let $n':=(n(1+\epsilon)\pi^*)^{1/5}$. We set 
\[
\mathscr{M}_S =\{(M)\in{\mathds{M}}^k_{\delta, \pi^*}: \mu= p^* \text{ and }M=M_n({p'}), \text{ where } p'\in B_{k-1}(1/n')\}.
\]
Noting that the uniform distribution over $\mathscr{M}_S$, $U(\mathscr{M}_S)$, is induced by $U(B_{k-1}(1/n'))$, the uniform distribution over $B_{k-1}(1/n')$ and thus is well-defined.

An important property of the above construction is that for a sample sequence $X^n\sim (M)\in \mathscr{M}_S$, $N_k$,  the number of times that state $k$ appears in $X^n$, is a binomial random variable with parameters $n$ and $\pi^*$. Therefore, Lemma~\ref{lemma9} implies that $N_k$ is highly concentrated around its mean $n\pi^*$. 

\subsection{\texorpdfstring{$L_2$}{L2}-divergence lower bound}
Let us first consider the $L_2$-distance. Similar to Lemma~\ref{lemma1}, $\hat{M}^*$, the estimator that minimizes $\EE_{(M)\sim U(\mathscr{M}_S)} [\varepsilon_n^{L_2}(M,\hat{M})]$, can be computed exactly. In particular, we have the following lemma.

\begin{Lemma}\label{lemma10}
There exists an estimator $\hat{M}^*$ with
\[
\hat{M}^*_{x^n}(i,\cdot) = p^*, \forall i\in[k-1],
\]
and
\[
\hat{M}^*_{x^n}(k,k) = \pi^*,
\]
such that $\hat{M}^*$ minimizes $\EE_{(M)\sim U(\mathscr{M}_S)} [\varepsilon_n^{L_2}(M,\hat{M})]$.
\end{Lemma}

Based on the above lemma, we can relate the minimax estimation risk of Markov chains to the minimax prediction risk of \iid\ processes. For simplicity, denote $\mathscr{B}_{\iid}:=\{(p)\in \mathds{IID}^{k-1}: p\in B_{k-1}(1/n')\}$. The following lemma holds. 
\begin{Lemma}\label{lemma11}
For any $x^n\in[k]^n$, let $\mathds{I}(x^n)$ be the collection of indexes $j\in[n]$ such that $x_j=k$. Then, 
\begin{align*}
&{\EE}_{(M)\sim U(\mathscr{M}_S)} [\EE_{X^n\sim (M)}[L_2(M(k,\cdot), \hat{M}^*_{X^n}(k,\cdot))\indic_{\mathds{I}(X^n)=\mathds{I}_0}]]\\
&=C{(\mathds{I}_0, \pi^*, p^*, n)}\min_{\hat{P}}\EE_{P\sim U(\mathscr{B}_{\iid})} [\rho_{|\mathds{I}_0|}^{L_2}(P,\hat{P})],
\end{align*}
where $\mathds{I}_0$ is an arbitrary non-empty subset of $[n]$ and $C{(\mathds{I}_0, \pi^*, p^*, n)}$ is a constant whose value only depends on $\mathds{I}_0, \pi^*, p^*, \text{ and } n$.
\end{Lemma}

\begin{proof}
We first consider the inner expectation on the left-hand side of the equality. For any $(M)\in \mathscr{M}_S$, we have
\begin{align*}
&\EE_{X^n\sim (M)}[{L_2}(M(k,\cdot), \hat{M}^*_{X^n}(k,\cdot))\indic_{\mathds{I}(X^n)=\mathds{I}_0}]\\
&=\sum_{x^n: {\mathds{I}(x^n)=\mathds{I}_0}}P(x^n){L_2}(M(k,\cdot), \hat{M}^*_{x^n}(k,\cdot))\\
&=\sum_{x^n: {\mathds{I}(x^n)=\mathds{I}_0}}\mu(x_1)\prod_{t=1}^{n-1} M(x_t, x_{t+1}) {L_2}(M(k,\cdot), \hat{M}^*_{x^n}(k,\cdot)).
\end{align*}
Let us partition $\mathds{I}_0$ into two parts: the collection of indexes $m\in\mathds{I}_0\cap[n-1]$ such that $m\in \mathds{I}_0$ and $m+1\not\in \mathds{I}_0$, say $\{m_1,\ldots,m_s\}$, and the remaining elements in $\mathds{I}_0$. By the construction of $\mathscr{M}_S$, we have
\begin{align*}
&\sum_{x^n: {\mathds{I}(x^n)=\mathds{I}_0}}\mu(x_1)\prod_{t=1}^{n-1} M(x_t, x_{t+1}) {L_2}(M(k,\cdot), \hat{M}^*_{x^n}(k,\cdot))\\
&=(\pi^*)^{|\mathds{I}_0|}\Paren{\frac{\bar{\pi}^*}{k-1}}^{n-s-|\mathds{I}_0|} \sum_{x^n: {\mathds{I}(x^n)=\mathds{I}_0}}\prod_{t=1}^{s} M(k, x_{m_{t}+1}) {L_2}(M(k,\cdot), \hat{M}^*_{x^n}(k,\cdot)).
\end{align*}
For any $x^n$, let $x^n\setminus \mathds{I}_0$ denote the subsequence $x_{j_1},\ldots, x_{j_{n-|\mathds{I}_0|-s}}$ such that $j_1<j_2\ldots<j_{n-|\mathds{I}_0|-s}$, $j_t\not\in \mathds{I}_0 \text{ and } j_t-1\not\in \{m_1,\ldots,m_s\}, \forall t$. We can further partition the last summation according to $x^n\setminus \mathds{I}_0$ as follows.
\begin{align*}
& \sum_{x^n: {\mathds{I}(x^n)=\mathds{I}_0}}\prod_{t=1}^{s} M(k, x_{m_{t}+1}) {L_2}(M(k,\cdot), \hat{M}^*_{x^n}(k,\cdot))\\
&=\sum_{y^{n-|\mathds{I}_0|-s}\in[k-1]^{n-|\mathds{I}_0|-s}}\Paren{\sum_{\substack{x^n: {x_{j}=k, \forall j\in\mathds{I}_0}\\ \text{ and } x^n\setminus \mathds{I}_0=y^{n-|\mathds{I}_0|-s}}}\prod_{t=1}^{s} M(k, x_{m_{t}+1}) {L_2}(M(k,\cdot), \hat{M}^*_{x^n}(k,\cdot))}.
\end{align*}
Fixing $y^{n-|\mathds{I}_0|-s}\in [k-1]^{n-|\mathds{I}_0|-s}$, there is a bijective mapping from $S(\mathds{I}_0,y^{n-|\mathds{I}_0|-s}):=\{x^n: {x_{j}=k, \forall j\in\mathds{I}_0} \text{ and } x^n\setminus \mathds{I}_0=y^{n-|\mathds{I}_0|-s}\}$ to $[k-1]^{s}$, say $g(\cdot)$. Furthermore, we have $\hat{M}^*(k,k)=\pi^*$. Hence, we can denote $q^*_{g(x^n)}:=\frac{\hat{M}^*_{x^n}(k,[k-1])}{\bar{\pi}^*}$ for $x^n\in S(\mathds{I}_0,y^{n-|\mathds{I}_0|-s})$ and treat it as a mapping from $[k-1]^s$ to $\Delta_{k-1}$.
Also, $(M)\in \mathscr{M}_S$ implies that $M(k,[k-1])=p'$ for some $p'\in B_{k-1}(1/n')$. Thus, 
\[
{L_2}(M(k,\cdot), \hat{M}^*_{x^n}(k,\cdot))=(\bar{\pi}^*)^2 L_2(p', q^*_{g(x^n)}),
\]
\[
\prod_{t=1}^{s} M(k, x_{m_{t}+1}) L_2(M(k,\cdot), \hat{M}^*_{x^n}(k,\cdot))= \prod_{t=1}^{s} p'(x_{m_{t}+1}) (\bar{\pi}^*)^2 L_2(p', q^*_{g(x^n)}),
\]
and
\begin{align*}
&\sum_{x^n\in S(\mathds{I}_0,y^{n-|\mathds{I}_0|-s})}\prod_{t=1}^{s} M(k, x_{m_{t}+1}) L_2(M(k,\cdot), \hat{M}^*_{x^n}(k,\cdot))\\
&=\sum_{x^n\in S(\mathds{I}_0,y^{n-|\mathds{I}_0|-s})}\prod_{t=1}^{s} p'(x_{m_{t}+1}) (\bar{\pi}^*)^2 L_2(p', q^*_{g(x^n)})\\
&=\sum_{z^s\in[k-1]^s}\prod_{t=1}^{s} p'(z_t) (\bar{\pi}^*)^2 L_2(p', q^*_{z^s})\\
&=\EE_{Z^s\sim (p')} [(\bar{\pi}^*)^2 L_2(p', q^*_{Z^s})],
\end{align*}
where $(p')$ is an \iid\ process whose underlying distribution is $p'$. 

By definition, $\hat{M}^*$ minimizes $\EE_{(M)\sim U(\mathscr{M}_S)} [\varepsilon_n^{L_2}(M,\hat{M})]$ and for each $x^n\in[k]^n$, its value $\hat{M}^*_{x^n}$ is completely determined by $x^n$. Besides, $\{S(\mathds{I}_0,y^{n-|\mathds{I}_0|-s}): \mathds{I}_0\subset[n] \text{ and } y^{n-|\mathds{I}_0|-s}\in [k-1]^{n-|\mathds{I}_0|-s}\}$ forms a partition of $[k]^n$. Therefore, by the linearity of expectation and the definition of $q^*$, the estimator $q^*$ also minimizes $\EE_{p'\sim U(B_{k-1}(1/n'))}[\EE_{Z^s\sim (p')} [(\bar{\pi}^*)^2 L_2(p', q_{Z^s})]]$, where the minimization is over all the possible mappings $q$ from $[k-1]^s$ to $\Delta_{k-1}$. Equivalently, we have
\[
\EE_{p'\sim U(B_{k-1}(1/n'))}[ \EE_{Z^n\sim (p')} [(\bar{\pi}^*)^2 L_2(p', q^*_{Z^n})]]=\min_{\hat{P}}\EE_{P\sim U(\mathscr{B}_{\iid})} (\bar{\pi}^*)^2 [\rho_{s}^{L_2}(P,\hat{P})].
\]
This immediately yields the lemma.
\end{proof}

For any $(M)\in \mathscr{M}_S$, denote by $N_k((M),n)$ the number of times that state $k$ appears in $X^{n}\sim (M)$, which is a random variable induced by $(M)$ and $n$. Lemma~\ref{lemma11}, we can deduce that 
\begin{Lemma}\label{lemma12}
\[
\min_{\hat{M}}\EE_{(M)\sim U(\mathscr{M}_S)} [\varepsilon_n^{L_2}(M,\hat{M})]\geq{\EE_{(M)\sim U(\mathscr{M}_S)}\left[ (\bar{\pi}^*)^2\min_{\hat{P}}\EE_{P'\sim U(\mathscr{B}_{\iid})} [\rho_{N_k((M),n)}^{L_2}(P',\hat{P})]\right]}.
\]
\end{Lemma}

By Lemma~\ref{lemma9} and our construction of $\mathscr{M}_S$, the probability that $N_k((M),n)\geq (1+\epsilon) n\pi^*$ is at most $\exp(-\epsilon^2n\pi^*/3)$ for any $(M)\in \mathscr{M}_S$ and $\epsilon\in(0,1)$. This together with Lemma~\ref{lemma12} and 
\[
\min_{\hat{P}}\EE_{P\sim U(\mathscr{B}_{\iid})} [\rho_{m}^{L_2}(P,\hat{P})]\gtrsim{\frac{1-\frac{1}{k-1}}{(1+\epsilon)n\pi^*}}, \forall m< (1+\epsilon)n\pi^*,
\]
from~\cite{KamathOPS15} yields
\begin{Lemma}\label{lemma13}
For all $\epsilon\in(0,1)$,
\[
\varepsilon_n^{L_2}(\mathscr{M})=\varepsilon_n^{L_2}({\mathds{M}}^k_{\delta, \pi^*})\gtrsim {\frac{(1-\frac{1}{k-1})(1-\pi^*)^2}{n\pi^*(1+\epsilon)}}.
\]
\end{Lemma}

\subsection{Lower bound for ordinary \texorpdfstring{$f$}{f}-divergences}
Now we proceed from the $L_2$-distance to ordinary $f$-divergences. The following lemma from~\cite{KamathOPS15} shows that $D_f(p,q)$ decreases if we move $q$ closer to $p$.
\begin{Lemma}\label{lemma14}
For $p_1>q_1$, $p_2<q_2$ and $d\leq\min\{p_1-q_1,q_2-p_2\}$,
\[
q_1 f\Paren{\frac{p_1}{q_1}}+q_2 f\Paren{\frac{p_2}{q_2}}\geq (q_1+d)f\Paren{\frac{p_1}{q_1+d}}+(q_2-d)f\Paren{\frac{p_2}{q_2-d}}.
\]
\end{Lemma}
Based on the above lemma, we show that for any $x^n\in[k]^n$, the value of the optimal estimator is always close to $(u_{k-1}\bar{\pi}^*,\pi^*)$.

Let $\hat{p}^*_{x^n}:=\hat{M}^*_{x^n}(k,\cdot)$. For any $x^n\in[k]^n$, we claim that either $\hat{p}^*_{x^n}(j)\geq (\frac{1}{k-1}-\frac{1}{n'})\bar{\pi}^*$, $\forall j\in[k-1]$ and $\hat{p}^*_{x^n}(k)\geq{\pi^*}$ OR $\hat{p}^*_{x^n}(j)\leq (\frac{1}{k-1}+\frac{1}{n'})\bar{\pi}^*$, $\forall j\in[k-1]$ and $\hat{p}^*_{x^n}(k)\leq{\pi^*}$. Otherwise, Lemma~\ref{lemma14} implies that we can reduce the estimation risk by moving $\hat{p}^*_{x^n}$ closer to $(u_{k-1}\bar{\pi}^*, \pi^*)$.

If $\hat{p}^*_{x^n}(j)\geq (\frac{1}{k-1}-\frac{1}{n'})\bar{\pi}^*$, $\forall j\in[k-1]$ and $\hat{p}^*_{x^n}(k)\geq{\pi^*}$, then $\hat{p}^*_{x^n}(j)\leq (\frac{1}{k-1}+\frac{k-2}{n'})\bar{\pi}^*$, $\forall j\in[k-1]$ and $\hat{p}^*_{x^n}(k)\leq{\pi^*+\frac{k-1}{n'}\bar{\pi}^*}$. Similarly, if $\hat{p}^*_{x^n}(j)\leq (\frac{1}{k-1}+\frac{1}{n'})\bar{\pi}^*$, $\forall j\in[k-1]$ and $\hat{p}^*_{x^n}(k)\leq{\pi^*}$, then $\hat{p}^*_{x^n}(j)\geq (\frac{1}{k-1}-\frac{k-2}{n'})\bar{\pi}^*$, $\forall j\in[k-1]$ and $\hat{p}^*_{x^n}(k)\geq{\pi^*-\frac{k-1}{n'}\bar{\pi}^*}$. 

Now we relate $D_f(p,\hat{p}^*)$ to $L_2(p,\hat{p}^*)$. For simplicity, denote $p:=M(k,\cdot)$ and drop $x^n$ from $\hat{p}^*_{x^n}$. 

\begin{Lemma}\label{lemma15}
For sufficiently large $n$,
\[
D_f(p,\hat{p}^*)\asymp \frac{(k-1)f''(1)}{2} L_2(p,\hat{p}^*).
\]
\end{Lemma}
\begin{proof}
By the previous lemma, $\hat{p}^*_{x^n}(j)= (\frac{1}{k-1}\pm\frac{k-2}{n'})\bar{\pi}^*$, $\forall j\in[k-1]$ and $\hat{p}^*_{x^n}(k)={\pi^*\pm\frac{k-1}{n'}\bar{\pi}^*}$. Therefore,
\[
\frac{p(i)}{\hat{p}^*(i)}\in \left[ \frac{n'-\frac{k}{\delta}}{n'+\frac{k}{\delta}}, \frac{n'+\frac{k}{\delta}}{n'-\frac{k}{\delta}}\right], \forall i\in[k].
\]
Let us denote the interval on the right hand side by $I$.

For sufficiently large $n$, we can apply the second-order Taylor expansion to $f$ at point $1$.
\begin{align*}
D_f(p,\hat{p}^*)
&=\sum_{i\in [k]} \hat{p}^*(i)f \Paren{\frac{p(i)}{\hat{p}^*(i)}}\\
&=\sum_{i\in [k]} \left(\hat{p}^*(i)\Paren{\frac{p(i)}{\hat{p}^*(i)}-1}f'(1) + \frac{\hat{p}^*(i)}{2}\Paren{\frac{p(i)}{\hat{p}^*(i)}-1}^2 f''(1)\right.\\&\left.\pm\frac{\hat{p}^*(i)}{6}\left|\frac{p(i)}{\hat{p}^*(i)}-1\right|^3 \max_{z\in I}|f'''(z)|\right)\\
&=\sum_{i\in [k]} \left(\frac{\hat{p}^*(i)}{2}\Paren{\frac{p(i)}{\hat{p}^*(i)}-1}^2 f''(1)\pm\frac{1}{6}\frac{k}{n'}\Paren{\frac{p(i)}{\hat{p}^*(i)}-1}^2 \max_{z\in I}|f'''(z)|\right)\\
&\gtrsim \frac{f''(1)}{2} \sum_{i\in [k-1]} \hat{p}^*(i)\Paren{\frac{p(i)}{\hat{p}^*(i)}-1}^2\\
& \asymp \frac{(k-1)f''(1)}{2\bar{\pi}^*} L_2(p,\hat{p}^*).
\end{align*}
\end{proof}

Lemma~\ref{lemma15} together with Lemma~\ref{lemma13} yields
\begin{Lemma}\label{lemma16}
For sufficiently large $n$,
\[
\varepsilon_n^{f}({\mathds{M}}^k_{\delta, \pi^*})\gtrsim {(1-\pi^*)\frac{(k-2)f''(1)}{2n\pi^*}}.
\]
\end{Lemma}

\section{Minimax estimation: upper bound}\label{sec4}
\subsection{Concentration of the counts}
The proof of the upper bound relies on the following concentration inequality, which shows that for any Markov chain in ${\mathds{M}}^k_{\delta}$ and any state $i\in[k]$, with high probability $N_i$ stays close to $(n-1)\pi_i$, for sufficiently large $n$.
\begin{Lemma}\label{lemma17}
Given a sample sequence $X^n$ from any Markov chain $(M)\in{\mathds{M}}^k_{\delta}$, let $N_{i}$ denote the number of times that symbol $i$ appears in $X^{n-1}$. Then for any $t\geq 0$,
\[
\textrm{Pr}(|N_i-(n-1)\pi_i|>t)\leq \sqrt{\frac{2}{\delta}}\exp\left( \frac{-t^2/C(\delta)}{4((n-1)+2C(\delta))+40t} \right),
\]
where $\pi$ is the stationary distribution of $(M)$ and
\[
C(\delta):= \Ceil{-\frac{\ln4}{\ln{(1-\delta)}}+1}.
\]
\end{Lemma}

\begin{proof}
Given $(M)\in{\mathds{M}}^k_{\delta}$, recall that $P^{n+1}$ denotes the distribution of $X_{n+1}$ if we draw $X^{n+1}\sim (M)$. First, we show that 
\[
D_{L_1}(P^{n+1},\pi)\leq 2(1-\delta)^n. 
\]
Let $\Pi$ be the $k\times k$ matrix such that $\Pi(i,\cdot)=\pi$ for all $i\in[k]$. Noting that $M(i,j)\geq\delta\Pi(i,j)$, we can 
define
\[
M_\delta := \frac{M-\delta\Pi}{1-\delta},
\]
which is also a valid transition matrix.

By induction, we can show
\[
M^n = (1-(1-\delta)^n)\Pi + (1-\delta)^n M_{\delta}^n.
\]

Let us rearrange the terms:
\[
M^n-\Pi = (1-\delta)^n (M_{\delta}^n-\Pi).
\]
Hence, let $|\cdot|$ denote the $L_1$ norm, we have
\[
D_{L_1}(P^{n+1},\pi) = |\mu(M^n-\Pi)|= |(1-\delta)^n \mu(M_{\delta}^n-\Pi)|\leq 2(1-\delta)^n.
\]

This implies that we can upper bound $t_{mix}$ by $C(\delta)$.

The remaining proof follows from Proposition 3.4, Theorem 3.4, and Proposition 3.10 of~\cite{conc15} and is omitted here for the sake of brevity. 
\end{proof}

Noting that $\textrm{Pr}(|N_i-(n-1)\pi_i|>(n-1))=0$, we have
\[
\textrm{Pr}(|N_i-(n-1)\pi_i|>t)\leq \sqrt{\frac{2}{\delta}}\exp\left( \frac{-t^2}{4C(\delta)(11(n-1)+2C(\delta))} \right).
\]
Informally, we can express the above inequality as
\[
\textrm{Pr}(|N_i-(n-1)\pi_i|>t)\leq \Theta_\delta(\exp(\Theta_\delta(-t^2/n))),
\]
which is very similar to the Hoeffding's inequality for the \iid\ processes. As an important implication, the following lemma bounds the moments of $|N_i-(n-1)\pi_i|$.
\begin{Lemma}\label{lemma18}
For $N_i$ defined in Lemma~\ref{lemma17} and any $m\in{\mathds{Z}^+}$,
\[
\EE[|N_i-(n-1)\pi_i|^m]\leq \frac{m\Gamma(m/2)}{\sqrt{2\delta}} (4C(\delta)(11(n-1)+2C(\delta)))^{m/2} .
\]
\end{Lemma}
\begin{proof}
The statement follows from
\begin{align*}
\EE[|N_i-(n-1)\pi_i|^m]
&= \int\limits_0^{\infty} \Pr(|N_i-(n-1)\pi_i|^m>t)\,dt\\
&= \int\limits_0^{\infty} \Pr(|N_i-(n-1)\pi_i|>t^{1/m})\,dt\\
&\leq \sqrt{\frac{2}{\delta}}\int\limits_0^{\infty} \exp\left( \frac{-t^{2/m}}{4C(\delta)(11n+2C(\delta))} \right) \,dt\\
&=  \frac{m}{\sqrt{2\delta}} (4C(\delta)(11n+2C(\delta)))^{m/2}\int\limits_0^{\infty}e^{-y} y^{m/2-1} \,dy\\
&= \frac{m\Gamma(m/2)}{\sqrt{2\delta}} (4C(\delta)(11n+2C(\delta)))^{m/2}.
\end{align*}
\end{proof}
\subsection{A modified add-\texorpdfstring{$\beta$}{b} estimator}
The difficulty with analyzing the performance of the original add-$\beta$ estimator is that the chain's initial distribution could be far away from its stationary distribution and finding a simple expression for $\EE[N_i]$ and $\EE[N_{ij}]$ could be hard. To overcome such difficulty, we ignore the first few sample points and construct a new add-$\beta$ estimator based on the remaining sample points. To be more specific, let $X^n$ be a length-$n$ sample sequence drawn from the Markov chain $(M)$. Removing the first $m$ sample points, $X^n_{m+1}:=X_{m+1},\ldots, X_n$ can be viewed as a length-$(n\!\!-\!\!m)$ sample sequence drawn from $(M)$ whose initial distribution $\mu'$ satisfies
\[
{L_1}(\mu', \pi)<2(1-\delta)^{m-1}.
\]
Setting $m=\sqrt{n}$, we have ${L_1}(\mu', \pi)\lesssim1/n^2$. Noting that $\sqrt{n}\ll n$ for sufficiently large $n$, without loss of generality, we assume that the original distribution $\mu$ already satisfies ${L_1}(\mu, \pi)<1/n^2$. If not, we can simply replace $X^n$ by $X^n_{\sqrt{n}+1}$.

To prove the upper bound, we consider the following (modified) add-$\beta$ estimator:
\[
\hat{M}^{+\beta}_{X^n}(i,j) := \frac{N_{ij}+{\beta}}{N_{i}+k\beta},\ \forall i,j\in[k],
\]
where $\beta>0$ is a fixed constant.

We can compute the expected values of these counts as
\begin{align*}
\EE[N_i] &= (n-1)\pi_i+\sum_{t=1}^{n-1} (\EE[\indic_{X_t=i}]-\pi_i)\\&=(n-1)\pi_i\pm \mathcal{O}(1/(n^2\delta))
\end{align*}
and 
\begin{align*}
\EE[N_{ij}] &= (n-1)\pi_i M_{ij}+\sum_{t=1}^{n-1} (\EE[\indic_{X_t=i}\indic_{X_{t+1}=j}]-\pi_i M_{ij})\\
&= (n-1)\pi_i M_{ij}+\sum_{t=1}^{n-1} (\EE[\indic_{X_t=i}]-\pi_i)M_{ij}\\
&= (n-1)\pi_i M_{ij}\pm \mathcal{O}( 1/(n^2\delta)).
\end{align*}

\subsection{Analysis}
For notational convenience, let us re-denote $n':=n-1$.

By Lemma~\ref{lemma17},
\[
\textrm{Pr}\left(\left|N_i-{n'\pi_i}\right|>t\right)\leq \Theta_\delta(\exp(\Theta_\delta(-t^2/n)))
\]
and
\[
\textrm{Pr}\left(\left|N_{ij}-{n'\pi_iM_{ij}}\right|>t\right)\leq \Theta_\delta(\exp(\Theta_\delta(-t^2/n))).
\]
The second inequality follows from the fact that $N_{ij}$ can be viewed as the sum of counts from the following two Markov chains over $[k]\times[k]$ whose transition probabilities are greater than $\delta^2$:
\[
(X_{1},X_{2}), (X_{3},X_{4}), \ldots
\]
and
\[
(X_{2},X_{3}), (X_{4},X_{5}), \ldots.
\]

In other words, $N_i$ and $N_{ij}$ are highly concentrated around ${n'\pi_i}$ and $n'\pi_iM_{ij}$, respectively. Let $A_{i}$ denote the event that $N_i=n'\pi_i(1\pm\delta/2)$ and $N_{ij}=n'\pi_i M_{ij}(1\pm\delta/2)$, $\forall j\in[k]$. Let $A_i^C$ denote the event that $A_i$ does not happen. Applying the union bound, we have
\[
\EE[\indic_{A_i^C}]=\Pr(A_i^C)\leq \Theta_\delta(\exp(\Theta_\delta(-n))).
\]

Now consider
\[
D_f(p, q)=\sum_{i\in[k]}q(i)f\Paren{\frac{p(i)}{q(i)}},
\]
the corresponding estimation risk of $\hat{M}^{+\beta}$ over a particular sate $i\in[k]$ can be decomposed as
\[
\EE[D_f(M({i,\cdot}),\hat{M}^{+\beta}_{X^n}(i,\cdot))\indic_{A_{i}}]+\EE[D_f(M({i,\cdot}),\hat{M}^{+\beta}_{X^n}(i,\cdot))\indic_{A_{i}^C}].
\]
Noting that 
\[
\hat{M}^{+\beta}_{X^n}(i,j) = \frac{N_{ij}+{\beta}}{N_{i}+k\beta}\in \left[\frac{\beta}{n+k\beta}, 1\right]
\]
 and $M_{ij}\in[\delta,1]$, we have
\[
|D_f(M({i,\cdot}),\hat{M}^{+\beta}_{X^n}(i,\cdot))|\leq k \cdot \frac{n+\beta}{k\beta}\cdot\max_{y\in [\delta, k+n/\beta]} f(y).
\]
Hence, we can bound the second term as
\begin{align*}
\EE[D_f(M({i,\cdot}),\hat{M}^{+\beta}_{X^n}(i,\cdot))\indic_{A_{i}^C}]
&\leq \frac{n+\beta}{\beta}\cdot\max_{y\in [\delta, k+n/\beta]} f(y) \cdot \EE[\indic_{A_{i}^C}]
\\&\leq  \frac{n+\beta}{\beta}\cdot\max_{y\in [\delta, k+n/\beta]} f(y) \cdot \Theta_\delta(\exp(\Theta_\delta(-n)))
\\&=\frac{o(1)}{n},
\end{align*}
where the last step follows from our assumption that $f$ is sub-exponential.

By the definition of $D_f$ and $\hat{M}^{+\beta}$, 
\[
\EE\left[D_f(M({i,\cdot}),\hat{M}^{+\beta}_{X^n}(i,\cdot))\indic_{A_{i}}\right]
 = \EE\left[ \sum_{j\in[k]}\frac{N_{ij}+{\beta}}{N_{i}+k\beta} f\Paren{\frac{M_{ij}}{\frac{N_{ij}+{\beta}}{N_{i}+k\beta}}}\indic_{A_{i}}\right].
\]
Let  $h(x):=f\Paren{\frac{1}{1+x}}$, then $h$ is thrice continuously differentiable around some neighborhood of point $0$ and
\[
f(x)=h\Paren{\frac{1}{x}-1}.
\] 
We apply Taylor expansion to $h$ at point $0$ and rewrite the expectation on the right-hand side as
\begin{align*}
\EE \sum_{j\in[k]}\frac{N_{ij}+{\beta}}{N_{i}+k\beta} f\Paren{\frac{M_{ij}}{\frac{N_{ij}+{\beta}}{N_{i}+k\beta}}}\indic_{A_{i}}
&=
\EE \sum_{j\in[k]}\frac{N_{ij}+{\beta}}{N_{i}+k\beta} h\Paren{\frac{(N_{ij}-M_{ij}N_{i})+{\beta(1-kM_{ij})}}{M_{ij}(N_{i}+k\beta)}}\indic_{A_{i}}\\
&=
\EE \sum_{j\in[k]} \frac{N_{ij}+{\beta}}{N_{i}+k\beta} \left[h'(0)\frac{(N_{ij}-M_{ij}N_{i})+{\beta(1-kM_{ij})}}{M_{ij}(N_{i}+k\beta)}\right.\\& \left.+\frac{h''(0)}{2}\Paren{\frac{(N_{ij}-M_{ij}N_{i})+{\beta(1-kM_{ij})}}{M_{ij}(N_{i}+k\beta)}}^2\right.\\& \left.
\pm\frac{M(\delta)}{6}\left| \frac{(N_{ij}-M_{ij}N_{i})+{\beta(1-kM_{ij})}}{M_{ij}(N_{i}+k\beta)} \right|^3\right]\indic_{A_{i}},
\end{align*}
where by our definition of $A_i$, we set
\[
M(\delta):=\max_{z\in\left[-\frac{2\delta}{1-\delta}, \frac{2\delta}{1-\delta}\right]}|h'''(z)|.
\]

Now, we bound individual terms.
Taking out $h'(0)$, the first term evaluates to:
\begin{align*}
&\EE\sum_{j\in[k]}\frac{N_{ij}+{\beta}}{N_{i}+k\beta} \frac{(N_{ij}-M_{ij}N_{i})+{\beta(1-kM_{ij})}}{M_{ij}(N_{i}+k\beta)}\indic_{A_{i}}\\
&=
\EE\sum_{j\in[k]}((N_{ij}-n'\pi_i M_{ij})+(n'\pi_i M_{ij}+{\beta})) \frac{(N_{ij}-M_{ij}N_{i})+{\beta(1-kM_{ij})}}{M_{ij}(N_{i}+k\beta)^2}\indic_{A_{i}}\\
&=
 \EE\sum_{j\in[k]} \frac{(N_{ij}-n'\pi_i M_{ij})}{M_{ij}} \frac{N_{ij}-n'\pi_i M_{ij})+(n'\pi_i M_{ij}-M_{ij}N_{i})+{\beta(1-kM_{ij})}}{(N_{i}+k\beta)^2}\indic_{A_{i}}\\
&+
\frac{(n'\pi_i M_{ij}+{\beta})}{M_{ij}} \frac{(N_{ij}-M_{ij}N_{i})+{\beta(1-kM_{ij})}}{(N_{i}+k\beta)^2}\indic_{A_{i}}\\
&=
 \EE\sum_{j\in[k]} \frac{(N_{ij}-n'\pi_i M_{ij})}{M_{ij}} \frac{(N_{ij}-n'\pi_i M_{ij})}{(N_{i}+k\beta)^2}+\frac{(N_{ij}-n'\pi_i M_{ij})(n'\pi_i -N_{i})}{(N_{i}+k\beta)^2}\\
&+
n'\pi_i \frac{(N_{ij}-M_{ij}N_{i})+{\beta(1-kM_{ij})}}{(N_{i}+k\beta)^2}+\frac{o(1)}{n}\\
&=-\EE\frac{(N_{i}-n'\pi_i)^2}{(N_{i}+k\beta)^2}
 +\EE\sum_{j\in[k]}\frac{1}{M_{ij}}\frac{(N_{ij}-n'\pi_i M_{ij})^2}{(N_{i}+k\beta)^2}
 +\frac{o(1)}{n}\\
&=-\EE\frac{(N_{i}-n'\pi_i)^2}{(n'\pi_i+k\beta)^2}
 +\EE\sum_{j\in[k]}\frac{1}{M_{ij}}\frac{(N_{ij}-n'\pi_i M_{ij})^2}{(n'\pi_i+k\beta)^2}
 +\frac{o(1)}{n}.
\end{align*}

Taking out $h''(0)/2$, the second term evaluates to: 
\begin{align*}
&\EE\sum_{j\in[k]}\frac{N_{ij}+{\beta}}{N_{i}+k\beta} \Paren{\frac{(N_{ij}-M_{ij}N_{i})+{\beta(1-kM_{ij})}}{M_{ij}(N_{i}+k\beta)}}^2\indic_{A_{i}}\\
&=
\EE\sum_{j\in[k]}((N_{ij}-M_{ij}N_{i})+(M_{ij}N_{i}+{\beta})) \frac{\Paren{(N_{ij}-M_{ij}N_{i})+{\beta(1-kM_{ij})}}^2}{M_{ij}^2(N_{i}+k\beta)^3}\indic_{A_{i}}\\
&=
\EE\sum_{j\in[k]}(N_{ij}-M_{ij}N_{i}) \frac{\Paren{(N_{ij}-M_{ij}N_{i})+{\beta(1-kM_{ij})}}^2}{M_{ij}^2(N_{i}+k\beta)^3}\indic_{A_{i}}\\
&+(M_{ij}N_{i}+{\beta}) \frac{\Paren{(N_{ij}-M_{ij}N_{i})+{\beta(1-kM_{ij})}}^2}{M_{ij}^2(N_{i}+k\beta)^3}\indic_{A_{i}}\\
&=
\EE\sum_{j\in[k]}(M_{ij}N_{i}+{\beta}) \frac{\Paren{(N_{ij}-M_{ij}N_{i})+{\beta(1-kM_{ij})}}^2}{M_{ij}^2(N_{i}+k\beta)^3}+
\frac{o(1)}{n}\\
&=\EE\sum_{j\in[k]}\frac{1}{M_{ij}}\frac{(N_{ij}-M_{ij}N_{i})^2}{(N_{i}+k\beta)^2}+2\EE\sum_{j\in[k]}(M_{ij}N_{i}+{\beta}) \frac{(N_{ij}-M_{ij}N_{i}){\beta(1-kM_{ij})}}{M_{ij}^2(N_{i}+k\beta)^3}+\frac{o(1)}{n}\\
&=\EE\sum_{j\in[k]}\frac{1}{M_{ij}}\frac{(N_{ij}-n'M_{ij}\pi_i+n'M_{ij}\pi_i-M_{ij}N_{i})^2}{(N_{i}+k\beta)^2}+\frac{o(1)}{n}\\
&=-\EE\frac{(N_{i}-n'\pi_i)^2}{(N_{i}+k\beta)^2}+\EE\sum_{j\in[k]}\frac{1}{M_{ij}}\frac{(N_{ij}-n'M_{ij}\pi_i)^2}{(N_{i}+k\beta)^2}+\frac{o(1)}{n}\\
&=-\EE\frac{(N_{i}-n'\pi_i)^2}{(n'\pi_i+k\beta)^2}
 +\EE\sum_{j\in[k]}\frac{1}{M_{ij}}\frac{(N_{ij}-n'\pi_i M_{ij})^2}{(n'\pi_i+k\beta)^2}
 +\frac{o(1)}{n}.
\end{align*}
\vspace{+2em}

Finally, taking out ${M(\delta)}/{6}$, the last term can be bounded as
\begin{align*}
&\EE \sum_{j\in[k]} \frac{N_{ij}+{\beta}}{N_{i}+k\beta}\left| \frac{(N_{ij}-M_{ij}N_{i})+{\beta(1-kM_{ij})}}{M_{ij}(N_{i}+k\beta)} \right|^3\indic_{A_{i}}\\
&\leq 4 \sum_{j\in[k]} \frac{ \EE\left|N_{ij}-M_{ij}N_{i} \right|^3+\left|{\beta(1-kM_{ij})} \right|^3}{M_{ij}^3(n'\pi_i(1-\delta/2)+k\beta)^3}\indic_{A_{i}}\\
&\leq4 \sum_{j\in[k]} \frac{ 4\EE\left|N_{ij}-M_{ij}n'\pi_i\right|^3+4M_{ij}^3\EE\left|n'\pi_i-N_{i} \right|^3+\left|{\beta(1-kM_{ij})} \right|^3}{M_{ij}^3(n'\pi_i(1-\delta/2)+k\beta)^3}\indic_{A_{i}}\\
&=\frac{o(1)}{n},
\end{align*}
where we have used the ineuqality $(a+b)^3\leq 4(|a|^3+|b|^3)$ twice.

By the definition of $h(\cdot)$, we have
\[
h'(0)=-f'(0)
\]
and
\[
\frac{h''(0)}{2}=f'(0)+\frac{f''(0)}{2}.
\]
Hence, consolidating all the previous results,
\begin{align*}
&\EE[D_f(M({i,\cdot}),\hat{M}^{+\beta}_{X^n}(i,\cdot))]\\
&=\frac{f''(0)}{2(n'\pi_i+k\beta)^2}\EE\Paren{-(N_{i}-n'\pi_i)^2
 +\sum_{j\in[k]}\frac{1}{M_{ij}}(N_{ij}-n'\pi_i M_{ij})^2}+\frac{o(1)}{n}\\
&=\frac{f''(0)}{2(n'\pi_i+k\beta)^2}\Paren{-\EE N_{i}^2
 +\sum_{j\in[k]}\frac{1}{M_{ij}}\EE N_{ij}^2}+\frac{o(1)}{n}.
\end{align*}

It remains to analyze $\EE N_{i}^2$ and $\EE N_{ij}^2$.

For $\EE N_{i}^2$, we have
\begin{align*}
\EE N_{i}^2
&=\EE\Paren{\sum_{t< n} \indic_{X_{t}=i}}^2\\
&=\EE\Paren{\sum_{t< n} \indic_{X_{t}=i} } +2\EE\Paren{\sum_{t_1<t_2< n} \indic_{X_{t_1}=i}\indic_{X_{t_2}=i}}\\
&=\sum_{t< n} \Pr(X_{t}=i)+2\sum_{t_1<t_2< n} \Pr(X_{t_1}=i)\Pr(X_{t_2}=i|X_{t_1}=i)\\
&=n'\pi_i+\mathcal{O}(1)+2\sum_{t_1<t_2<n} \Paren{\pi_i \pm \mathcal{O}\Paren{\frac{1}{n^2}}}\Pr(X_{t_2}=i|X_{t_1}=i)\\
&=n'\pi_i+\mathcal{O}(1)+2\pi_i\sum_{t_1<t_2< n} \Pr(X_{t_2}=i|X_{t_1}=i)\\
&=n'\pi_i+\mathcal{O}(1)+2\pi_i\sum_{t_1<t_2< n} \sum_{j\in[k]}\Pr(X_{t_2}=i|X_{t_1+1}=j)\Pr(X_{t_1+1}=j|X_{t_1}=i)\\
&=n'\pi_i+\mathcal{O}(1)+2\pi_i\sum_{j\in[k]}\sum_{t_1<t_2< n} \Pr(X_{t_2}=i|X_{t_1+1}=j)M_{ij}.
\end{align*}

For $\EE N_{ij}^2$, we have
\begin{align*}
\EE N_{ij}^2
&=\EE\Paren{\sum_{t<n} \indic_{X_{t}=i} \indic_{X_{t+1}=j}}^2\\
&=\EE\Paren{\sum_{t<n} \indic_{X_{t}=i} \indic_{X_{t+1}=j}} +2\EE\Paren{\sum_{t_1<t_2<n} \indic_{X_{t_1}=i}\indic_{X_{t_1+1}=j}\indic_{X_{t_2}=i}\indic_{X_{t_2+1}=j}}\\
&=M_{i,j}\sum_{t<n} \Pr(X_{t}=i)+2\sum_{t_1<t_2<n} \Pr(X_{t_1}=i)M_{ij}\Pr(X_{t_2}=i|X_{t_1+1}=j)M_{ij}\\
&=M_{i,j}n'\pi_i+\mathcal{O}(1)+2\sum_{t_1<t_2<n} \Paren{\pi_i \pm \mathcal{O}\Paren{\frac{1}{n^2}}}\Pr(X_{t_2}=i|X_{t_1+1}=j)M_{ij}^2\\
&=M_{i,j}n'\pi_i+\mathcal{O}(1)+2\pi_iM_{ij}^2\sum_{t_1<t_2<n} \Pr(X_{t_2}=i|X_{t_1+1}=j).
\end{align*}

Thus, the desired quantity evaluates to
\begin{align*}
-\EE N_{i}^2 +\sum_{j\in[k]}\frac{1}{M_{ij}}\EE N_{ij}^2
 &=\sum_{j\in[k]}\Paren{n'\pi_i+\mathcal{O}(1)+2\pi_iM_{ij}\sum_{t_1<t_2<n} \Pr(X_{t_2}=i|X_{t_1+1}=j)}\\
 &-\Paren{n'\pi_i+\mathcal{O}(1)+2\pi_i\sum_{j\in[k]}\sum_{t_1<t_2< n} \Pr(X_{t_2}=i|X_{t_1+1}=j)M_{ij}}\\
 &\leq (k-1)n'\pi_i+\mathcal{O}(k).
 \end{align*}

The above inequality yields
\begin{align*}
&\EE[D_f(M({i,\cdot}),\hat{M}^{+\beta}_{X^n}(i,\cdot))]\\
&=\frac{f''(0)}{2(n'\pi_i+k\beta)^2}\EE\Paren{-(N_{i}-n'\pi_i)^2
 +\sum_{j\in[k]}\frac{1}{M_{ij}}(N_{ij}-n'\pi_i M_{ij})^2}+\frac{o(1)}{n}\\
&\lesssim\frac{(k-1)f''(0)}{2n\pi_i}.
\end{align*}
This completes our proof for ordinary $f$-divergences.

\subsection{\texorpdfstring{$L_2$}{L2}-divergence upper bound}
Finally, we consider the $L_2$-divergence. Again, we assume that the sample sequence $X^n\sim (M)$ and $\mu$ satisfies
\[
D_{L_1}(\pi, \mu)<\frac{1}{n^2}.
\]
Instead of using an add-constant estimator, we use the following add-${\sqrt{N_{i}}}/{k}$ estimator:
\[
\hat{M}^{+{\sqrt{N_{i}}}/{k}}_{X^n}(i,j) := \frac{N_{ij}+{\sqrt{N_{i}}}/{k}}{N_{i}+{\sqrt{N_{i}}}},\ \forall i,j\in[k]\times [k].
\]
Now, consider the expected loss for a particular state $i\in[k]$.
\begin{align*}
\EE\sum_{j\in[k]}\Paren{M_{ij}-\frac{N_{ij}+{\sqrt{N_{i}}}/{k}}{N_{i}+{\sqrt{N_{i}}}}}^2
&=\sum_{j\in[k]}\EE\Paren{\frac{(M_{ij}N_i-N_{ij})+\sqrt{N_{i}}(M_{ij}-{1}/{k})}{N_{i}+{\sqrt{N_{i}}}}}^2\\
&=\sum_{j\in[k]}\EE\Paren{\frac{M_{ij}N_i-N_{ij}}{N_{i}+{\sqrt{N_{i}}}}}^2+\Paren{\frac{\sqrt{N_{i}}(M_{ij}-{1}/{k})}{N_{i}+{\sqrt{N_{i}}}}}^2\\&+2\EE\frac{(M_{ij}N_i-N_{ij})(\sqrt{N_{i}}(M_{ij}-{1}/{k}))}{\Paren{N_{i}+{\sqrt{N_{i}}}}^2}.
\end{align*}

We first show that the last term is negligible. Noting that
\begin{align*}
\EE\sum_{j\in[k]}\frac{(M_{ij}N_i-N_{ij})(\sqrt{N_{i}}(M_{ij}-{1}/{k}))}{\Paren{N_{i}+{\sqrt{N_{i}}}}^2}
&=\EE\sum_{j\in[k]}\frac{(M_{ij}N_i-N_{ij})M_{ij}}{\sqrt{N_{i}}\Paren{\sqrt{N_{i}}+1}^2},
\end{align*}
we can apply Taylor expansion to the function 
\[
f(x):=\frac{1}{\sqrt{x}(\sqrt{x}+1)^2}
\]
at point $x=\EE[N_{i}]$ and set $x=N_i$:
\[
f(x)=f(\EE[N_{i}])+f'(N_{i}')(N_i-\EE[N_{i}]),
\]
where $N_i'\in[\EE[N_{i}], N_i]$. Hence,
\begin{align*}
&\EE\sum_{j\in[k]}\frac{(M_{ij}N_i-N_{ij})M_{ij}}{\Paren{N_{i}+{\sqrt{N_{i}}}}\Paren{\sqrt{N_{i}}+1}}\\
&=\EE\sum_{j\in[k]}\Paren{f(\EE[N_{i}])+f'(N_{i}')(N_i-\EE[N_{i}])}(M_{ij}N_i-N_{ij})M_{ij}\\
&=\EE\sum_{j\in[k]}\frac{(M_{ij}N_i-N_{ij})M_{ij}}{\sqrt{\EE[N_{i}]}(\sqrt{\EE[N_{i}]}+1)^2}+\frac{-3\sqrt{N_{i}'}-1}{2(\sqrt{N_{i}'}+1)^3(N_{i}')^{3/2}}(N_i-\EE[N_{i}])(M_{ij}N_i-N_{ij})M_{ij}\\
&\leq\EE\sum_{j\in[k]}\mathcal{O}\Paren{\frac{1}{n^{7/2}}}+\frac{-3\sqrt{N_{i}'}-1}{2(\sqrt{N_{i}'}+1)^3(N_{i}')^{3/2}}M_{ij}\sqrt{\EE(N_i-\EE[N_{i}])^2\EE(M_{ij}N_i-N_{ij})^2}\\
&=\Theta\Paren{\frac{1}{n^{3/2}}}.
\end{align*}
where the last step follows from Lemma~\ref{lemma18}. It remains to consider
\[
\EE\Paren{\frac{M_{ij}N_i-N_{ij}}{N_{i}+{\sqrt{N_{i}}}}}^2={\frac{\EE(M_{ij}N_i-N_{ij})^2}{(n\pi_i+{\sqrt{n\pi_i}})^2}}+\frac{o(1)}{n}.
\]
According to the previous derivations, for $M_{ij}^2\EE N_{i}^2$, we have
\begin{align*}
M_{ij}^2\EE N_{i}^2
&=M_{ij}^2\sum_{t< n} \Pr(X_{t}=i)+2M_{ij}^2\sum_{t_1<t_2< n} \Pr(X_{t_1}=i)\Pr(X_{t_2}=i|X_{t_1}=i).
\end{align*}

For $\EE N_{ij}^2$, we have
\begin{align*}
\EE N_{ij}^2
&=M_{ij}\sum_{t<n} \Pr(X_{t}=i)+2M_{ij}^2\sum_{t_1<t_2<n} \Pr(X_{t_1}=i)\Pr(X_{t_2}=i|X_{t_1+1}=j).
\end{align*}

For 2$M_{ij}\EE N_{ij} N_i$, we have
\begin{align*}
2M_{ij}\EE N_{ij}N_i
&=2M_{ij}\EE\Paren{\sum_{t<n} \indic_{X_{t}=i} \indic_{X_{t+1}=j}}\Paren{\sum_{t< n} \indic_{X_{t}=i}}\\
&=2M_{ij}\EE\Paren{\sum_{t<n} \indic_{X_{t}=i} \indic_{X_{t+1}=j}} + 2M_{ij}\EE\Paren{\sum_{t_1<t_2< n} \indic_{X_{t_1}=i} \indic_{X_{t_1+1}=j}\indic_{X_{t_2}=i}}\\
&+2M_{ij}\EE\Paren{\sum_{t_2<t_1<n} \indic_{X_{t_1}=i} \indic_{X_{t_1+1}=j}\indic_{X_{t_2}=i}}\\
&=2M_{ij}^2\sum_{t<n} \Pr(X_{t}=i)+2M_{ij}^2\sum_{t_1<t_2< n}\Pr(X_{t_2}=i|X_{t_1+1}=j)\Pr(X_{t_1}=i)\\
&+2M_{ij}^2\sum_{t_2<t_1<n}\Pr(X_{t_1}=i|X_{t_2}=i)\Pr(X_{t_2}=i).
\end{align*}

Therefore,
\[
\EE\Paren{M_{ij}N_i-N_{ij}}^2=M_{ij}(1-M_{ij})n\pi_i+\frac{o(1)}{n}.
\]

Finally,
\[
\EE\sum_{j\in [k]}\Paren{\frac{\sqrt{N_{i}}(M_{ij}-{1}/{k})}{N_{i}+{\sqrt{N_{i}}}}}^2=\frac{o(1)}{n}+{\frac{-\frac{1}{k}\EE[N_i]+\EE[N_i]\sum_{j\in[k]}M_{ij}^2}{(n\pi_i+{\sqrt{n\pi_i}})^2}}.
\]
We have
\begin{align*}
\EE\sum_{j\in[k]}\Paren{M_{ij}-\frac{N_{ij}+{\sqrt{N_{i}}}/{k}}{N_{i}+{\sqrt{N_{i}}}}}^2
&=\Paren{1-\frac{1}{k}} \frac{1}{n\pi_i}+\frac{o(1)}{n}.
\end{align*}
This completes our proof for the $L_2$-divergence.
\end{document}